\documentclass[11pt,a4paper]{article}

\usepackage[a4paper,textwidth=166mm,top=25mm,bottom=25mm,headheight=14pt,headsep=9pt,footskip=19pt]{geometry}
\usepackage{setspace}
\usepackage{needspace}
\usepackage{silence}
\usepackage[T1]{fontenc}
\usepackage[utf8]{inputenc}
\usepackage{tgpagella}
\usepackage{microtype}
\usepackage{amsfonts,amssymb,amsthm,amsmath}
\usepackage{newpxmath}
\usepackage{fancyhdr}
\usepackage{graphicx,booktabs}
\usepackage{tikz}
\usetikzlibrary{arrows.meta,positioning}
\usepackage{pgfplots}
\usepgfplotslibrary{groupplots}
\pgfplotsset{compat=1.18}
\usepackage{enumitem}
\usepackage{xcolor}
\definecolor{linkblue}{HTML}{174C73}
\usepackage[colorlinks,linkcolor=linkblue,citecolor=linkblue,urlcolor=linkblue]{hyperref}
\usepackage[depth=2]{bookmark}
\hypersetup{
  pdftitle={Beyond Global Divergences: A Local-Mass Perspective on Bayesian Inference},
  pdfauthor={Hanli Xu, Fengxiang He, Sarat Moka},
  pdfdisplaydoctitle=true
}
\usepackage{comment}
\usepackage{float}
\usepackage{subcaption}
\usepackage{placeins}
\usepackage{bm}
\usepackage{caption}

\newcommand{\R}{\mathbb{R}}

\newcommand{\MI}{\mathrm{MI}}
\newcommand{\pa}{\boldsymbol\theta}

\def\tr(#1){\mathrm{trace}(#1)}
\def\Exp(#1){\mathbb E(#1)}
\def\Exps(#1){\mathbb E\sparen(#1)}

\def\inpr#1,#2{\langle #1 , #2 \rangle}
\def\ip<#1,#2>{\langle #1,#2 \rangle}

\def\paren(#1){\left( #1 \right)}

\def\sparen(#1){\Bigl ( #1 \Bigr )}

\newtheoremstyle{compactplain}
 {7pt plus 2pt minus 2pt}
 {7pt plus 2pt minus 2pt}
 {\itshape}
 {}
 {\bfseries}
 {.}
 {0.5em}
 {}

\newtheoremstyle{compactdefinition}
 {7pt plus 2pt minus 2pt}
 {7pt plus 2pt minus 2pt}
 {\normalfont}
 {}
 {\bfseries}
 {.}
 {0.5em}
 {}

\theoremstyle{compactplain}
\newtheorem{theorem}{Theorem}[section]
\newtheorem{lemma}[theorem]{Lemma}
\newtheorem{corollary}[theorem]{Corollary}

\newtheorem{assumption}{Assumption}

\theoremstyle{compactdefinition}
\newtheorem{definition}[theorem]{Definition}
\newtheorem{example}[theorem]{Example}

\numberwithin{equation}{section}
\definecolor{rulegray}{HTML}{68737D}
\DeclareCaptionLabelSeparator{rulebar}{\enspace\textcolor{rulegray}{\textbar}\enspace}
\fancypagestyle{firstpage}{
  \fancyhf{}
  
  \fancyfoot[R]{\footnotesize\thepage}
}
\begin{document}
\thispagestyle{firstpage}
\noindent{\color{rulegray}\rule{\linewidth}{0.6pt}}\par
\vspace{10pt}
{\raggedright\parindent=0pt
{\fontsize{20}{24}\selectfont\bfseries
Beyond Global Divergences:\\
A Local-Mass Perspective on Bayesian Inference\par}
\vspace{11pt}
{\normalsize\bfseries
Hanli Xu\textsuperscript{1},\enspace
Fengxiang He\textsuperscript{2},\enspace
Sarat Moka\textsuperscript{1}\par}
\vspace{3pt}
{\small\sffamily
\textsuperscript{1}UNSW Sydney \quad
\textsuperscript{2}University of Edinburgh\par
\href{mailto:z5664653@zmail.unsw.edu.au}{z5664653@zmail.unsw.edu.au}\quad
\href{mailto:fhe@ed.ac.uk}{fhe@ed.ac.uk}\quad
\href{mailto:s.moka@unsw.edu.au}{s.moka@unsw.edu.au}\par}
}

\vspace{15pt}
\pdfbookmark[1]{Abstract}{abstract}
{\setlength{\parindent}{0pt}\noindent
Global KL convergence and vanishing ELBO gaps need not preserve local small-ball scales. We study this gap using Mass Indices, which record power and logarithmic orders, and set-localised regularised extended KL (RE-KL), retaining singular components. We derive scale-preservation conditions and exponent shifts under power-log likelihoods and hard support constraints, and show that softening can change the limiting scale despite total variation convergence. Normalised RE-KL bounds yield scale preservation under a subcritical condition and directional one-sided comparisons under boundedness. Controlled experiments and comparisons of neural-network marginals with their priors, up to ResNet-50 on ImageNet, illustrate the local-mass analysis. Experiment code is included in the repository \url{https://github.com/Forsythia0604/Local-Mass-Framework}.
\par}
\vspace{4pt}

\section{Introduction}

Global objectives, such as Kullback--Leibler (KL) divergence \cite{kullback1951information} and the evidence lower bound (ELBO) \cite{blei2017variational}, are standard in Bayesian inference. However, global closeness does not necessarily determine how probability mass behaves near a specified parameter value. Recent variational methods seek better posterior mass coverage through inclusive-KL optimisation \cite{mcnamara_sequential_2024} and contrastive objectives \cite{ward2025softcvi}. Structured sparse Bayesian neural networks combine Dirac spikes with group shrinkage \cite{jantre2025spike}, while low-rank weight parametrisations can induce distributions supported on lower-dimensional sets \cite{toure2026singular}. These developments motivate our central question: when do Bayesian updating and posterior approximation preserve the power and logarithmic decay rates of small-ball probabilities around a fixed parameter value? Global KL convergence alone need not imply agreement in these scales (Theorem~\ref{thm:kl-mi-discontinuous}).

To address this issue, this paper develops a local-mass framework for Bayesian inference. The basic object of study is the small-ball probability \(p(B_r(\pa))\), \(r \downarrow 0\), where \(B_r(\pa)\) denotes the ball of radius \(r\) centred at a parameter value \(\pa\). We show how Bayesian updating can preserve or alter this local mass behaviour, and derive conditions under which an approximation retains the local scale of its target.

The \emph{Mass Index} summarises the rate at which the small-ball probability \(\mu(B_r(\pa)) \)
decays as \(r\downarrow 0\). For example, if \(\mu(B_r(\pa))\asymp r^a(\log(1/r))^b\), with \(a>0\), its power component is \(d/a\), where \(d\) is the ambient dimension, and its logarithmic component is \(b\), without dimension normalisation. The power component is a normalised reciprocal of the usual local-dimension exponent \cite{heurteaux2007dimension}. The logarithmic component distinguishes measures with the same leading power order, such as smooth and horseshoe-type priors near zero, and is related to logarithmic slowly varying corrections \cite{bingham1989regular}. The definitions use upper and lower order bounds and do not assume this power-log form.

To compare local masses, we apply the standard extended \(f\)-divergence construction \cite{liero2018optimal} to the unnormalised restrictions of a target \(p\) and an approximation \(q\). The resulting \emph{regularised extended KL} (RE-KL) retains the component of \(q\) singular with respect to \(p\). We track this contribution in bounds on the full local mass, without assuming \(q\ll p\). On the whole space, RE-KL has the Tsallis-\(\alpha\) form and converges to extended KL as \(\alpha\uparrow1\).

We show that the local effect of Bayesian updating is determined by the local scaling of the likelihood: a power-log likelihood factor shifts the corresponding prior small-ball exponents by its power and logarithmic exponents. Likelihoods bounded above and away from zero locally preserve both Mass Indices whenever defined.
When the likelihood imposes a hard support constraint \(S\), only the prior mass inside \(B_r(\pa)\cap S\) contributes to the posterior mass near \(\pa\). We compare this retained mass with the original prior mass in \(B_r(\pa)\). Under the stated assumptions, the power and logarithmic exponents of this retained fraction add to the corresponding prior small-ball exponents. We also show that smoothing hard constraints may not preserve this local structure: the zero-temperature limit of a softened likelihood and the small-neighbourhood limit need not commute.

We next ask when local RE-KL control preserves small-ball scales. Bounded reference-mass-normalised divergence gives an upper mass bound in the \(q\|p\) direction and a lower mass bound in the \(p\|q\) direction. If either normalised direction stays uniformly below the threshold \(1\), the two local masses are comparable and their power and logarithmic scales agree whenever defined. Thus the direction determines the one-sided bound, while the stronger subcritical condition gives scale preservation.

Controlled experiments first calibrate the Mass Index across regular, depleted, cusp, logarithmic, and atomic regimes. They illustrate the directional RE-KL result with an example in which one normalised direction remains bounded while the reverse direction diverges. On real datasets, Laplace approximations for Bayesian logistic regression on the UCI Breast Cancer, Iris \(0\) vs \(1\), and Wine \(0\) vs \(1\) datasets show that regular Bayesian updating can greatly increase posterior mass near the Laplace mean while preserving the leading local power order. Two further experiments carry the same quantities to Bayesian neural networks: per-weight variational posteriors of LeNet-5 on MNIST, and structured channel gates of ResNet-56 on CIFAR-10 and ResNet-50 on ImageNet. These compare pooled one-dimensional marginals with the prior. Atomic and continuous families have different asymptotic Power Mass Indices, while the finite-radius RE-KL curves illustrate their local mass differences.

\section{Related work}

\paragraph{Variational inference.}
Variational inference \cite{blei2017variational} approximates a target distribution within a tractable family, and the choice and direction of divergence \cite{minka2005divergence} influence the resulting approximation. R\'enyi \cite{li_renyi_2016} and $f$-divergence formulations \cite{wan2020f,zhang_variational_2024} further broaden the available objectives. Inclusive-KL methods such as Markovian score climbing \cite{naesseth2020markovian} and SMC-Wake \cite{mcnamara_sequential_2024} are motivated by mass-covering behaviour and the difficulty of gradient estimation. SoftCVI \cite{ward2025softcvi} uses a contrastive objective and demonstrates mass-covering behaviour empirically. Wan et al.\ \cite[Section~2.1]{wan2020f} formulate their divergence under \(q\ll p\), while Burt et al.\ \cite[Appendix~B]{burt2021understanding} use extended \(f\)-divergences with a singular boundary term in function-space variational inference. Our local bounds retain this term and compare the full small-ball masses and their scales, without assuming \(q\ll p\).

\paragraph{Local divergences.}
Avlogiaris et al.\ \cite{avlogiaris2016local} introduce local Csisz\'ar $\phi$-divergences for comparing distributions on selected regions. Bailey et al.\ \cite{bailey2022local} relate divergence limits between conditional local tail distributions to local intrinsic dimensionality. Censoring-based scoring rules \cite{depunder2026localizing} support region-specific forecast comparison while retaining the probability of the region. Our RE-KL analysis compares restricted measures without conditioning them on the neighbourhood. We derive directional sufficient conditions for equality or ordering of their power and logarithmic small-ball scales, when defined. These conditions concern local relations between fixed measures.

\paragraph{Local dimensions.}
The leading exponent of small-ball probability is classical in local dimension theory \cite{heurteaux2007dimension}, while regular variation \cite{bingham1989regular} separates power-law behaviour from slowly varying factors such as logarithmic corrections. When the local exponent is finite and positive, our Power Mass Index is its reciprocal multiplied by the ambient dimension, while the logarithmic scale distinguishes measures sharing the same leading exponent. Our definitions require neither regular variation nor exact power-log asymptotics.

\paragraph{Bayesian sparsity and compression.}
Bayesian compression uses sparsity-inducing priors to remove unnecessary network components \cite{louizos2017bayes}, while horseshoe priors support node selection through shrinkage \cite{ghosh_model_nodate}. Recent work combines Dirac spikes with group shrinkage for structurally sparse Bayesian neural networks \cite{jantre2025spike}. Low-rank weight parametrisations can also induce variational posteriors singular with respect to ambient Lebesgue measure \cite{toure2026singular}. Prior structure can affect predictive behaviour more broadly \cite{fortuin_bayesian_2022}. These settings motivate our comparison of local mass across priors, posteriors, and approximations. We study when this mass retains its power and logarithmic scales.

\paragraph{Posterior contraction and nonregular updating.}
Posterior contraction theory \cite{ghosal2000convergence} studies how priors and model complexity determine concentration around the truth as sample size grows, with small-ball probabilities playing a central role for Gaussian-process priors \cite{van2008rates} and structured sparse models \cite{castillo2015bayesian,polson_posterior_2018}. Recent results for heavy-tailed neural-network priors \cite{castillo2025posterior} establish adaptive contraction towards the true regression function for fractional posteriors and their mean-field variational approximations as sample size grows. We instead fix the data and study mass around a specified parameter value as the neighbourhood radius tends to zero. Parameter-dependent supports \cite{hirano2003asymptotic} also connect our setting to nonregular inference. We isolate how local likelihood scaling and support restrictions change the power and logarithmic mass scales, including cases where globally accurate soft approximations fail to preserve limiting local behaviour. 

\section{Preliminaries}

\paragraph{Measure-theoretic notation.}
For probability measures \(p\) and \(q\), \(p\ll q\) means that
\(q(E)=0\Rightarrow p(E)=0\) for every measurable set \(E\). We write
\(p\perp q\) if there exists measurable \(E\) such that
\(p(\mathbb R^d\setminus E)=0\) and \(q(E)=0\). These last two notions are referred
to as absolute continuity and mutual singularity, respectively.

Lebesgue decomposition separates the absolutely continuous and singular parts of one measure relative to another.

\begin{lemma}[Lebesgue decomposition; Theorem 3.8 \cite{folland1999real}] \label{lem:decom} For probability measures \(p\) and \(q\), there are unique measures \(q_{\ll p}\) and \(q_{\perp p}\) such that \(q_{\ll p}\ll p\), \(q_{\perp p}\perp p\), and \(q=q_{\ll p}+q_{\perp p}\). \end{lemma}

Total variation will measure global convergence when we soften a hard support constraint.

\begin{definition}[Total variation \cite{billingsley2013convergence}]
\label{def:tv}
Let \(p\) and \(q\) be probability measures on a measurable space
\((\Theta,\mathcal F)\). The total variation distance between \(p\) and \(q\) is
\(\|p-q\|_{\mathrm{TV}} := \sup_{E\in\mathcal F}|p(E)-q(E)|\).
\end{definition}

KL divergence provides the global benchmark for variational approximation.

\begin{definition}[KL divergence \cite{kullback1951information}]
\label{def:kl}
Let \(p\) and \(q\) be probability measures on a measurable space
\((\Theta,\mathcal F)\). If \(q\ll p\), the KL divergence from \(q\) to \(p\) is
\(D_{\mathrm{KL}}(q\|p) := \int_{\Theta} \log\!\left(\frac{\mathrm d q}{\mathrm d p}\right)\,\mathrm d q \).
Otherwise, we set \(D_{\mathrm{KL}}(q\|p)=\infty\).
\end{definition}

\paragraph{Problem formulation.}
Let \(\Theta=\mathbb R^d\) be the parameter space and \(\mathcal D=(z_1,\ldots,z_n)\) the observed data, with likelihood \(L(\vartheta\mid\mathcal D)\).

Let \(\pi_0\) be a prior probability measure on \(\Theta\). Whenever the marginal
\(
Z_{\mathcal D}
:=
\int_{\Theta} L(\vartheta\mid\mathcal D)\,\pi_0(\mathrm d\vartheta)
\) satisfies \(0<Z_{\mathcal D}<\infty\), the posterior is defined by
\(\pi_{\mathcal D}(E) =Z_{\mathcal D}^{-1} {\int_E L(\vartheta\mid\mathcal D)\,\pi_0(\mathrm d\vartheta)} \).
This map from the prior \(\pi_0\) to the posterior \(\pi_{\mathcal D}\) is called Bayesian updating.

When the target posterior is intractable, variational inference uses a tractable family \(\mathcal Q\subseteq\mathcal P(\Theta)\), where \(\mathcal P(\Theta)\) is the set of probability measures on \(\Theta\).

The variational approximation problem is to find $q^*$ such that
\(q^{*}\in \mathrm{argmin}_{q\in\mathcal Q} D_{\mathrm{KL}}(q\|\pi_{\mathcal D})\).
Or equivalently, maximising the evidence lower bound
\[
\mathrm{ELBO}_{\mathcal D}(q)
:=
\int_{\Theta}\log L(\vartheta\mid\mathcal D)\,q(\mathrm d\vartheta)
-
D_{\mathrm{KL}}(q\|\pi_0).
\]

\section{The local mass framework}
\label{sec:local framework}

Throughout, \(p\) and \(q\) are arbitrary probability measures on \(\Theta\).

\subsection{Mass Index: On local small-ball scale}

\label{subsec:mass}

The Mass Index summarises how small-ball mass decays as \(r\downarrow0\). Its power component is the reciprocal local power scale, normalised by \(d\); its logarithmic component records a correction after the power scale is fixed. Upper and lower comparisons allow for behaviour without regular variation or exact power-log asymptotics.

\begin{definition}[Power and Logarithmic Mass Indices ]\label{define:MI}
Let \(B_r(\pa)\) be the standard open ball of radius \(r\) centred at \(\pa\). For the power gauges below, we use \(\sup\varnothing=0\) and \(\inf\varnothing=+\infty\). For the logarithmic gauges, whose exponents range over \(\R\), we use \(\sup\varnothing=-\infty\) and \(\inf\varnothing=+\infty\).

\textbf{Power Mass Index.}
The sets \(\mathcal U_{\mathrm{pow}}\) and
\(\mathcal L_{\mathrm{pow}}\) collect the values of \(\eta\)
for which \(r^{1/\eta}\) gives an asymptotic upper or lower
bound on the small-ball mass, respectively.
As \(r\downarrow0\), define
\[
\begin{aligned}
\mathcal{U}_{\mathrm{pow}}
:=
\Big\{\eta>0:
p(B_r(\pa))=O(r^{1/\eta})
\Big\},\\
\mathcal{L}_{\mathrm{pow}}
:=
\Big\{\eta>0:
p(B_r(\pa))=\Omega(r^{1/\eta})
\Big\}.
\end{aligned}
\]
The upper and lower Power Mass Indices are
\(\overline{\MI}_{\mathrm{pow}}(p,\pa) :=d\inf\mathcal{U}_{\mathrm{pow}},\, \underline{\MI}_{\mathrm{pow}}(p,\pa) :=d\sup\mathcal{L}_{\mathrm{pow}}\).
If
$
\overline{\MI}_{\mathrm{pow}}(p,\pa)
=
\underline{\MI}_{\mathrm{pow}}(p,\pa),
$
then the common value is denoted by
$
\MI_{\mathrm{pow}}(p,\pa).
$

The upper and lower indices record the sharp thresholds for upper
and lower power bounds on the small-ball mass. They remain defined
even when no single Power Mass Index exists.
Example~\ref{exp:osc} illustrates this situation: the small-ball
mass has distinct power orders along different sequences of radii,
so the two indices do not coincide.

\textbf{Logarithmic Mass Index.}
Assume that \(\MI_{\mathrm{pow}}(p,\pa)\) exists and lies in \((0,\infty)\).
Let
$
a_{p,\pa}:=\frac{d}{\MI_{\mathrm{pow}}(p,\pa)}.
$
As $r\downarrow 0$, define 
\[
\begin{aligned}
&\mathcal{U}_{\mathrm{log}}
:=
\big\{
\eta\in\R:
p(B_r(\pa))=O\bigl(r^{a_{p,\pa}}(-\log r)^\eta\bigr)
\big\},\\
&\mathcal{L}_{\mathrm{log}}
:=
\big\{
\eta\in \R:
p(B_r(\pa))=\Omega\bigl(r^{a_{p,\pa}}(-\log r)^\eta\bigr)
\big\}.
\end{aligned}
\]
The upper and lower Logarithmic Mass Indices are
\(\overline{\MI}_{\mathrm{log}}(p,\pa) := \inf\mathcal{U}_{\mathrm{log}},\, \underline{\MI}_{\mathrm{log}}(p,\pa) := \sup\mathcal{L}_{\mathrm{log}}\).
If
$
\overline{\MI}_{\mathrm{log}}(p,\pa)
=
\underline{\MI}_{\mathrm{log}}(p,\pa),
$
then the common value is denoted by
$
\MI_{\mathrm{log}}(p,\pa).
$
\end{definition}

For eventually positive small-ball mass, let
\[
\underline a=\liminf_{r\downarrow0}\log p(B_r(\pa))/\log r
\qquad\text{and}\qquad
\overline a=\limsup_{r\downarrow0}\log p(B_r(\pa))/\log r
\]
be the lower and upper local dimensions \cite{heurteaux2007dimension}. The power thresholds give \(\overline\MI_{\mathrm{pow}}=d/\underline a\) and \(\underline\MI_{\mathrm{pow}}=d/\overline a\), with \(d/0=+\infty\) and \(d/(+\infty)=0\).

When both components are defined, we compare $(\MI_\mathrm{pow}(p,\pa),\MI_\mathrm{log}(p,\pa))$ lexicographically: first by the power component, then by the logarithmic component if the power components agree.
Throughout the paper, we denote 
$
\MI_{\mathrm{pow}}^{-1}(p,\pa)
:=
\frac{1}{\MI_{\mathrm{pow}}(p,\pa)}
$
whenever \(\MI_{\mathrm{pow}}(p,\pa)\in(0,\infty)\). 

\subsubsection{Mass Indices in four standard cases}

The next theorem calibrates the indices against familiar local behaviours.

\begin{theorem}[Mass Indices in four standard cases]
\label{prop:basic-calibrations}
Let \(p\) be a probability measure on \(\mathbb R^d\). For some $r_0>0$:
\begin{enumerate}[label=\textup{(\arabic*)},leftmargin=*,nosep]
\item \textup{Local hole.}
If \(p(B_{r_0}(\pa))=0\) then \(\MI_{\mathrm{pow}}(p,\pa)=0\).

\item \textup{Regular continuous mass.}
If \(p\) admits a continuous density \(\rho_p\) on \(B_{r_0}(\pa)\), and \(\rho_p(\pa)>0\), then \(\MI_{\mathrm{pow}}(p,\pa)=1\).

\item \textup{Power-log asymptotics.} If
\(p(B_r(\pa))\asymp r^a(-\log r)^b,\allowbreak \, r\downarrow 0\) with $a>0$, then $\MI_{\mathrm{pow}}(p,\pa)\allowbreak=da^{-1},\allowbreak\MI_{\mathrm{log}}(p,\pa)=b.$

\item \textup{Atomic mass.}
If \(p(\{\pa\})>0\) then \(\MI_{\mathrm{pow}}(p,\pa)=\infty\).
\end{enumerate}
\end{theorem}

A local hole, regular continuous mass, and an atom give power indices \(0,1,\infty\), respectively. The factor \(d\) makes ordinary full-dimensional scaling \(r^d\) the unit benchmark.

Table \ref{table:MI-common} gives representative one-dimensional examples at \(\pa=0\).

\begin{table}[t]
\centering
\caption{Power and Logarithmic Mass Indices at $0$.}
\label{table:MI-common}
\setlength{\tabcolsep}{4pt}
\renewcommand{\arraystretch}{1.15}
\begin{tabular}{@{}lcc@{}}
\hline
\textbf{Distribution}
& \(\MI_{\mathrm{pow}}\)
& \(\MI_{\mathrm{log}}\) \\
\hline
Gaussian
& \(1\) & \(0\) \\

Student-\(t_\nu\) 
& \(1\) & \(0\) \\

Laplace
& \(1\) & \(0\) \\

Horseshoe
& \(1\) & \(1\) \\

Spike-and-slab
& \(\infty\) & - \\
\hline
\end{tabular}
\end{table}

\subsubsection{Scope of the definition}

Gaussian and horseshoe priors share power index \(1\), but the horseshoe's additional logarithmic factor gives log index \(1\). Horseshoe priors are used in Bayesian compression \cite{ghosh_model_nodate}.

Power indices \(0\) and \(\infty\) also cover finer behaviours outside the polynomial scale, including super-polynomial and logarithmic decay. The Power Mass Index therefore does not distinguish all such behaviours.

\subsection{RE-KL: A set-localised divergence}

\label{subsec:rekl}

Avlogiaris et al.\ \cite{avlogiaris2016local} studied local Csisz\'ar divergences, and P\'oczos et al.\ \cite{poczos2011estimation} studied global Tsallis-\(\alpha\) divergences and their estimation. We use a Tsallis-type generator in the extended \(f\)-divergence construction on a measurable set \(E\), retaining both absolutely continuous and singular contributions to local mass.

\begin{definition}[Regularised extended KL]
For \(\alpha\in(0,1)\), define
\(f_\alpha(u) :=\frac {u^\alpha-\alpha u+\alpha-1}{{\alpha-1}}, \, u\in[0,\infty)\).
The Lebesgue decomposition (Lemma \ref{lem:decom}) of $q$ with respect to $p$ is
$
q=q_{\ll p}+q_{\perp p}.
$
Then the RE-KL divergence on any measurable set $E$ is defined as
\[
D_\alpha^E(q\|p):=\underbrace{\int_E\,f_\alpha\left(\frac{\mathrm dq_{\ll p}}{\mathrm dp}\right)\,\mathrm dp}_\text{absolutely continuous part}+\underbrace{\frac{\alpha}{1-\alpha}q_{\perp p}(E)}_\text{singular penalty}.
\]
\end{definition}

\subsubsection{Relation to standard divergences}

We first relate RE-KL on the whole space to familiar global
divergences. For arbitrary probability measures \(p\) and \(q\),
the measure \(\nu=p+q\) dominates both. Thus the densities
\(u=\mathrm dq/\mathrm d\nu\) and \(v=\mathrm dp/\mathrm d\nu\)
allow us to state these relations without assuming \(q\ll p\).
On the whole space \(\Theta\), we have:

\noindent\textup{(1) KL divergence.} When $\alpha\uparrow 1$, 
\(D_\alpha^\Theta(q\|p)\uparrow D_{\mathrm{KL}}(q\|p)\).

\noindent\textup{(2) Hellinger distance $\mathcal H(q,p)$ \cite{beran1977minimum}.} When $\alpha=1/2$, \(D_{1/2}^\Theta(q\|p)\allowbreak=2\mathcal H^2(q,p)=2\bigl(1-\int_\Theta \sqrt{uv}\,\mathrm d\nu\bigr)\).

\noindent\textup{(3) Tsallis-$\alpha$ divergence \cite{poczos2011estimation}.} When $\alpha\allowbreak\in(0,1)$, 
$
D_{\alpha}^\Theta(q\|p) = (\alpha-1)^{-1}
\left(
\int_{\Theta} u^\alpha v^{1-\alpha}\,\mathrm d\nu\allowbreak
-1
\right).
$

For fixed \(\alpha<1\), the singular penalty is finite; any nonzero singular component instead makes KL infinite. Thus our local scale analysis uses a standard extended \(f\)-divergence, retaining singular mass when comparing restricted measures.

\subsection{Why global KL is not enough}
Global KL convergence can hide changes in small-ball scaling. GAN mode collapse \cite{lin2018pacgan} and incomplete mode coverage \cite{zhong2019rethinking} raise related concerns, but the following theorem addresses the local power order at a fixed parameter value.

\begin{theorem}[Sequential discontinuity of the Power Mass Index under KL convergence]
\label{thm:kl-mi-discontinuous}
Fix \(\pa\in\Theta\). Define
\(\mathcal M_{\pa}:= \{q\mkern-2mu\in\mkern-2mu\mathcal P(\Theta): \MI_{\mathrm{pow}}(q,\pa)\in(0,\infty)\}\).

Let \(p\in\mathcal M_{\pa}\). Then, for every
\(\kappa\in(0,\infty)\) with \(\kappa\neq \MI_{\mathrm{pow}}(p,\pa)\), there exists a sequence
\(\{q_n\}_{n\ge1}\subseteq \mathcal M_{\pa}\) such that
$
D_{\mathrm{KL}}(q_n\|p)\to0,
$
but
$
\MI_{\mathrm{pow}}(q_n,\pa)=\kappa
\,\allowbreak
\text{for every } n\ge1.
$

In particular, the functional
$
\Phi_{\pa}:\mathcal M_{\pa}\to(0,\infty),$ defined as 
$
\Phi_{\pa}(q):=\MI_{\mathrm{pow}}(q,\pa)
$
is not sequentially continuous at \(p\) under global KL convergence.
\end{theorem}

\paragraph{Interpretation.}
The theorem shows that global variational accuracy alone need not preserve local mass scales. It does not assert that enlarging a variational family necessarily causes distortion.

For a fixed model and dataset,
\(\log Z_{\mathcal D}-\mathrm{ELBO}_{\mathcal D}(q)=D_{\mathrm{KL}}(q\|\pi_{\mathcal D})\)
whenever well-defined \cite{blei2017variational}. For \(p=\pi_{\mathcal D}\), Theorem~\ref{thm:kl-mi-discontinuous} therefore shows that a vanishing ELBO gap need not preserve the posterior's Power Mass Index.

\section{Local mass in Bayes}
\label{sec:theory}

We apply the tools of Section~\ref{sec:local framework} to Bayesian updating and approximation. Complete proofs are in Appendix~\ref{app:proofs}.

\subsection{Bayesian updating and local mass scaling}

We consider two ways likelihoods can change local posterior mass: power-log scaling near \(\pa\), and hard constraints that exclude some nearby prior mass.

\subsubsection{Stability under Bayesian updating}

We begin with two assumptions that ensure a well-defined
posterior and a well-defined prior Power Mass Index.

\begin{assumption}[Finite marginal]\label{as:bayes}
The marginal \(Z_{\mathcal D}:=\int_\Theta L(\vartheta\mid\mathcal D)\,\pi_0(\mathrm d\vartheta)\) satisfies \(0<Z_{\mathcal D}<\infty\).
\end{assumption}

\begin{assumption}[Well-defined Power Mass Index]\label{as:existence}
For a probability measure \(\mu\) on \(\Theta\), \(\MI_{\mathrm{pow}}(\mu,\pa)\) exists.
\end{assumption}

The first assumption makes the posterior well-defined. The second requires the upper and lower Power Mass Indices to agree. It excludes oscillation between distinct power orders (Example~\ref{exp:osc}), so the following results can track a single Power Mass Index. It holds in all four cases of Theorem~\ref{prop:basic-calibrations}.

\begin{theorem}[Local likelihood scaling]\label{thm:local-likelihood-scaling}
Under Assumption~\ref{as:bayes}, suppose Assumption~\ref{as:existence} holds for \(\mu=\pi_0\), with \(\MI_{\mathrm{pow}}(\pi_0,\pa)\in[0,\infty)\). For some \(\gamma,\beta\in\R\), \(r_0\in(0,1)\), and \(0<c_L<C_L<\infty\), assume
\[
c_Lh_{\gamma,\beta}(\vartheta)\le L(\vartheta\mid\mathcal D)
\le C_Lh_{\gamma,\beta}(\vartheta)
\]
on \(B_{r_0}(\pa)\setminus\{\pa\}\), where \(h_{\gamma,\beta}(\vartheta):=\|\vartheta-\pa\|^\gamma(\log(1/\|\vartheta-\pa\|))^\beta\).

\textnormal{(i)} If \(0<\MI_{\mathrm{pow}}(\pi_0,\pa)<\infty\) and \(\gamma>-d\,\MI_{\mathrm{pow}}^{-1}(\pi_0,\pa)\), then the posterior power index exists and
\[
\MI_{\mathrm{pow}}^{-1}(\pi_{\mathcal D},\pa)
=\MI_{\mathrm{pow}}^{-1}(\pi_0,\pa)+d^{-1}\gamma.
\]
If the prior logarithmic index exists and is finite, so does the posterior index, with
\(\MI_{\mathrm{log}}(\pi_{\mathcal D},\pa) =\MI_{\mathrm{log}}(\pi_0,\pa)+\beta\).

\textnormal{(ii)} If \(\MI_{\mathrm{pow}}(\pi_0,\pa)=0\), then \(\MI_{\mathrm{pow}}(\pi_{\mathcal D},\pa)=0\).
\end{theorem}

\paragraph{Interpretation.}
Theorem~\ref{thm:local-likelihood-scaling} shows that local power-log likelihood behaviour shifts the small-ball power and logarithmic exponents by \(\gamma\) and \(\beta\), respectively. Thus the inverse Power Mass Index shifts by \(\gamma/d\). In part~\textnormal{(i)},
\(d\,\MI^{-1}_{\mathrm{pow}}(\pi_0,\pa)+\gamma>0\)
ensures local integrability; boundary cases are excluded.

More generally, if \(0<c_L\le L\le C_L<\infty\) on a full neighbourhood of \(\pa\), Bayes' formula gives \(c_L\pi_0(B_r)/Z_{\mathcal D}\le\pi_{\mathcal D}(B_r)\le C_L\pi_0(B_r)/Z_{\mathcal D}\) for sufficiently small \(r\), where \(B_r=B_r(\pa)\). Thus both Mass Indices are preserved whenever defined, including the power indices of atoms and local holes.

\subsubsection{Parameter-dependent support and surviving local mass}

For \(Y_i\stackrel{\mathrm{iid}}{\sim}\operatorname{Unif}(0,\vartheta)\) \cite{hirano2003asymptotic}, the likelihood is proportional to \(\vartheta^{-n}\mathbf1\{\vartheta\ge Y_{(n)}\}\), so the data restrict admissible parameters. More generally, consider
\begin{equation}\label{eq:g1}
L(\vartheta\mid\mathcal D)=g(\vartheta)\mathbf1_{S_{\mathcal D}}(\vartheta),
\end{equation}
where \(S_{\mathcal D}\) and \(g\ge0\) are measurable, and \(0<c_g\le g\le C_g<\infty\) on some \(B_{r_g}(\pa)\).

\begin{theorem}[Acceptance-ratio scaling of local Mass Indices]\label{thm:acceptance}
Under Assumption~\ref{as:bayes}, suppose \(0<\MI_{\mathrm{pow}}(\pi_0,\pa)<\infty\) and the likelihood satisfies \eqref{eq:g1} with the local bounds on \(g\) above. Define
\[
A_{S_{\mathcal D},\pa}(r):=
\frac{\pi_0(B_r(\pa)\cap S_{\mathcal D})}{\pi_0(B_r(\pa))}.
\]
If \(A_{S_{\mathcal D},\pa}(r)\asymp r^\xi(-\log r)^\lambda\) as \(r\downarrow0\), where either \(\xi>0\), or \(\xi=0\) and \(\lambda\le0\), then the posterior power index exists and
\[
\MI_{\mathrm{pow}}^{-1}(\pi_{\mathcal D},\pa)
=\MI_{\mathrm{pow}}^{-1}(\pi_0,\pa)+d^{-1}\xi.
\]
If the prior logarithmic index exists and is finite, so does the posterior index, with
\(\MI_{\mathrm{log}}(\pi_{\mathcal D},\pa) =\MI_{\mathrm{log}}(\pi_0,\pa)+\lambda\).
\end{theorem}

The surviving fraction \(A_{S_{\mathcal D},\pa}(r)\) multiplies the local prior mass up to constants. Its decay, rather than the bounded factor \(g\), therefore determines the change in scale. Smoothing the constraint need not preserve this effect.

\begin{corollary}[Softening and non-commuting local limits]\label{cor:softening-noncommuting}
Let \(s_\tau:\Theta\to(0,\infty)\) be continuous for every \(\tau>0\), with \(s_\tau\le C<\infty\) uniformly and \(s_\tau\to\mathbf1_{S_{\mathcal D}}\) \(\pi_0\)-almost everywhere as \(\tau\downarrow0\).

Assume \(\pi_0(S_{\mathcal D})>0\) and \(0<\MI_{\mathrm{pow}}(\pi_0,\pa)<\infty\). Write \(Z_\tau:=\int s_\tau\,\mathrm d\pi_0\), let \(\pi_\tau\) be the posterior for \(s_\tau\), and let \(\pi_{\mathcal D}\) be the posterior for \(\mathbf1_{S_{\mathcal D}}\). 

Then
\(\|\pi_\tau-\pi_{\mathcal D}\|_{\mathrm{TV}}\to0, \allowbreak\; \MI_{\mathrm{pow}}(\pi_\tau,\pa)=\MI_{\mathrm{pow}}(\pi_0,\pa)\).
Nevertheless, if \(A_{S_{\mathcal D},\pa}(r)\asymp r^\xi(-\log r)^\lambda\) with \(\xi>0\), then
\(\MI_{\mathrm{pow}}(\pi_{\mathcal D},\pa) \ne\lim_{\tau\downarrow0}\MI_{\mathrm{pow}}(\pi_\tau,\pa)\).
When the prior logarithmic index exists, \(\MI_{\mathrm{log}}(\pi_\tau,\pa)=\MI_{\mathrm{log}}(\pi_0,\pa)\) for every \(\tau>0\). 

If additionally the prior logarithmic index exists and is finite and \(A_{S_{\mathcal D},\pa}(r)\asymp(-\log r)^\lambda\), \(\lambda<0\), then
\(\MI_{\mathrm{log}}(\pi_{\mathcal D},\pa) \ne\lim_{\tau\downarrow0}\MI_{\mathrm{log}}(\pi_\tau,\pa)\).
\end{corollary}

At fixed temperature, the continuous positive soft factor is locally bounded above and away from zero, so it preserves the prior scale. The hard limit can retain only a vanishing fraction of local mass, so its small-ball scale can differ despite total variation convergence. Thus the limits \(r\downarrow0\) and \(\tau\downarrow0\) need not commute.

\subsection{Local RE-KL analysis for Mass Indices}

Fix \(\alpha\in(0,1)\). We first bound local mass errors, then compare asymptotic scales.

\begin{theorem}[Absolute error bound]\label{thm:local mass}
For probability measures \(p,q\) on \((\Theta,\mathcal F)\) and \(E\in\mathcal F\),
\[
|q(E)-p(E)|\le2\sqrt{C_\alpha D_\alpha^E(q\|p)},
\]
where \(C_\alpha=\max\{1,(1-\alpha)/\alpha\}\).
\end{theorem}

This bound converts local divergence rates into mass-error rates, including when \(q\not\ll p\).

\begin{corollary}\label{cor:local mass}
If \(D_\alpha^{B_r(\pa)}(q\|p)=O(r^\delta)\) for some \(\delta>0\), then \(|q(B_r(\pa))-p(B_r(\pa))|=O(r^{\delta/2})\).
\end{corollary}

Absolute error can be small even when the relative error on a shrinking ball is large. Normalising by the reference mass gives the required stronger control.

\begin{theorem}[Relative error bound]\label{thm:local multiplicative}
For probability measures \(p,q\) and a measurable set \(E\) with \(p(E)>0\), let \(\overline D_\alpha^E(q\|p):=\frac{D_\alpha^E(q\|p)}{p(E)}\). Then
\[
f_\alpha\!\left(\frac{q(E)}{p(E)}\right)\le\overline D_\alpha^E(q\|p),
\]
and
\[
\left|1-\frac{q(E)}{p(E)}\right|
\le\frac{2}{\alpha}\sqrt{\overline D_\alpha^E(q\|p)}
\sqrt{\overline D_\alpha^E(q\|p)+2}.
\]
\end{theorem}

The bound controls the full mass \(q(E)\), including its \(p\)-singular component. In particular, non-negativity of \(f_\alpha\) gives
\[
\frac{q_{\perp p}(E)}{p(E)}
\le \frac{1-\alpha}{\alpha}\,\overline D_\alpha^E(q\|p).
\]
Thus vanishing normalised divergence makes the singular mass negligible relative to \(p(E)\).

\begin{corollary}\label{cor:local multiplicative}
If \(p(B_r(\pa))>0\) for all sufficiently small \(r\) and \(\overline D_\alpha^{B_r(\pa)}(q\|p)=O(r^\delta)\) for some \(\delta>0\), then \(\left |1-q(B_r(\pa))/p(B_r(\pa))\right|=O(r^{\delta/2})\).
\end{corollary}

\paragraph{Directional implications.}
In variational approximation, \(p=\pi_{\mathcal D}\) is the target and \(q\) the approximation. The threshold below follows from \(f_\alpha(0)=1\): its sublevel sets below \(1\) are bounded away from both zero and infinity.

\begin{theorem}[Directional sufficient conditions for Power Mass Index comparison]\label{thm:one-sided-MI-pow}
Let \(p,q\) be probability measures with positive mass on \(B_r(\pa)\) for all sufficiently small \(r>0\), and suppose Assumption~\ref{as:existence} holds for both.

\textnormal{(i)} If \(\limsup_{r\to0}\overline D_\alpha^{B_r(\pa)}(q\|p)<1\), then \(\MI_{\mathrm{pow}}(q,\pa)=\MI_{\mathrm{pow}}(p,\pa)\). If \(\MI_{\mathrm{log}}(p,\pa)\) exists, then \(\MI_{\mathrm{log}}(q,\pa)=\MI_{\mathrm{log}}(p,\pa)\).

\textnormal{(ii)} If \(\limsup_{r\to0}\overline D_\alpha^{B_r(\pa)}(p\|q)<\infty\), then \(\MI_{\mathrm{pow}}(q,\pa)\ge\MI_{\mathrm{pow}}(p,\pa)\). If the power indices coincide and both logarithmic indices exist, then \(\MI_{\mathrm{log}}(q,\pa)\ge\MI_{\mathrm{log}}(p,\pa)\).
\end{theorem}

\paragraph{Interpretation.}
The direction determines which one-sided mass bound holds. A stronger condition gives both bounds.
By Theorem~\ref{thm:local multiplicative}, a finite limsup of the normalised divergence gives
\(q(B_r(\pa))\le C p(B_r(\pa))\) in the \(q\|p\) direction and
\(q(B_r(\pa))\ge c p(B_r(\pa))\) in the \(p\|q\) direction,
for fixed \(c,C>0\) and all sufficiently small \(r\). Either direction with limsup strictly below \(1\) is sufficient
for both bounds, hence equality of the power indices and of the logarithmic
indices whenever defined. This requires a uniform margin below \(1\), not
merely a value below \(1\) at each radius.

At \(\alpha=1/2\), the unnormalised RE-KL is symmetric; the directional
difference then comes from dividing by the respective reference masses.
These are conditional local implications, not a characterisation of
minimisers of the global variational objective.

Example~\ref{exp:dir} separates them: the \(p\|q\) quantity tends to \(1\), the reverse diverges, and the power indices are \(2\) for \(q\) and \(1\) for \(p\).
\section{Experiments}\label{sec:experiments}

Controlled examples and neural-network marginal comparisons illustrate the local-mass analysis.


\subsection{Implementation details}
Experiments~1--3 illustrate the calibration of
Mass Indices, preservation of the power scale under regular updating, and
directional control of local mass. Experiments~4--5 examine the same local
quantities in neural-network marginals. All curves are finite-radius
illustrations; asymptotic indices follow from the specified distribution
families.

\paragraph{Controlled settings.}
Experiment~1 compares regular, depleted, cusp, logarithmic, and atomic
small-ball curves. Experiment~2 uses Bayesian logistic regression on Breast
Cancer, Iris \(0\) vs \(1\), and Wine \(0\) vs \(1\), with a Gaussian prior
and a Laplace posterior approximation in \(d=5\). We measure mass around
the Laplace mean over five seeds per dataset. Experiment~3 compares
\(p(x)=1/2\) and \(q(x)=|x|^{-1/2}/4\) on \((-1,1)\), using normalised
local RE-KL at \(\alpha=3/4\).

\paragraph{Neural-network settings.}
We compare Gaussian, Student-\(t\), spike-and-slab, and hard-concrete
families using the atomic spike-and-slab prior
\(\pi_0=(1-\pi)\delta_0+\pi S\) as a common diagnostic reference.
The first three use KL regularisation against this prior; hard-concrete
uses expected-\(L_0\) regularisation. Experiment~4 trains LeNet-5 on
MNIST and studies weight marginals. Experiment~5 trains channel gates
on pretrained ResNet-56/CIFAR-10 and ResNet-50/ImageNet with frozen base
parameters. Gates act before BatchNorm, so a zero gate need not zero the
post-BatchNorm channel; running statistics still update. We evaluate pooled one-dimensional
small-ball mass and local RE-KL against \(\pi_0\),  using three ImageNet
seeds and one CIFAR-10 seed. Appendix~\ref{app:nn_protocol} gives the
training and evaluation details.

\subsection{Controlled experiments}
\begin{figure}[!htbp]
\centering
\includegraphics[width=.72\linewidth]{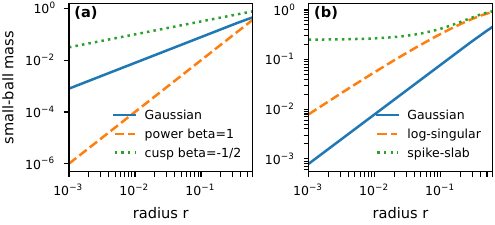}
\captionsetup{width=.72\linewidth}
\caption{Synthetic small-ball calibration across five local regimes:
regular mass, depletion, a cusp, a logarithmic correction, and an atom.}
\label{fig:synthetic_calibration}
\end{figure}

\begin{figure}[!htbp]
\centering
\includegraphics[width=.72\linewidth]{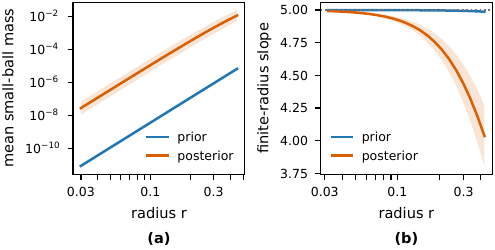}
\captionsetup{width=.72\linewidth}
\caption{UCI mean small-ball masses and finite-radius slopes, with standard
errors over 15 runs (three datasets and five seeds per dataset).}
\label{fig:uci_aggregate}
\end{figure}

\paragraph{Calibration and regular updating.}
Figure~\ref{fig:synthetic_calibration} separates the power, logarithmic,
and atomic regimes. Figure~\ref{fig:uci_aggregate} shows that the Laplace
posterior places substantially more mass near its mean than the prior,
while their small-radius slopes remain near \(d=5\). The mean posterior
slopes fitted over \(r\le0.06\) are \(4.969\), \(4.999\), and \(4.993\)
for Breast Cancer, Iris, and Wine, respectively; the corresponding prior
slopes are \(5.000\). Table~\ref{tab:uci_slopes} reports standard errors.
This agrees with local power-scale preservation for regular densities.
The decreasing slopes at larger radii are finite-radius effects.

\paragraph{Directional comparison.}
The analytic example in Figure~\ref{fig:nn_directional}(a) has
\(p(B_r)=r\) and \(q(B_r)=\sqrt r\), with power indices \(1\) and \(2\).
Its normalised local RE-KL tends to \(1\) in the \(p\|q\) direction and
diverges in the reverse direction. This illustrates the one-sided
comparison in Theorem~\ref{thm:one-sided-MI-pow}(ii).
Appendix~\ref{app:rekl_directionality} gives the exact expressions.

\subsection{Neural-network experiments}
\begin{figure}[!htbp]
\centering
\includegraphics[width=\textwidth]{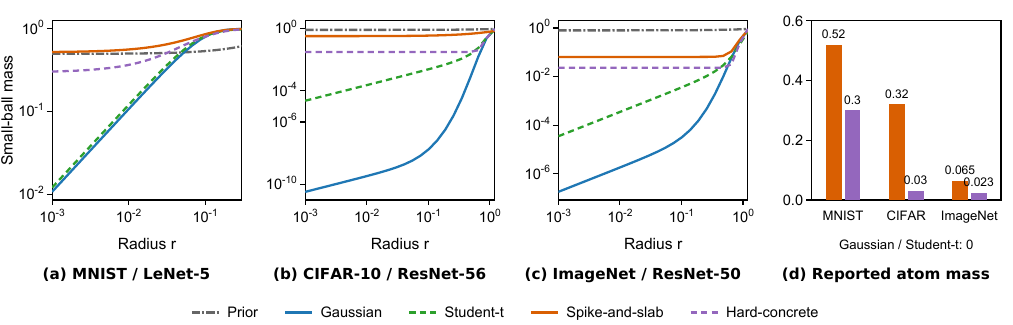}
\caption{Marginal local mass across neural networks. (a) LeNet-5/MNIST
weight marginals. (b) ResNet-56/CIFAR-10 and (c) ResNet-50/ImageNet channel
gates. (d) Reported zero-atom masses from Tables~\ref{tab:nn_mi}
and~\ref{tab:nn_resnet}; Gaussian and Student-\(t\) marginals have no
atom. ImageNet results average three seeds; CIFAR-10 uses one seed.
Hard-concrete curves preserve the zero atom and approximate the active
component by a Gaussian.
Each mass panel retains its own vertical scale.}
\label{fig:nn_smallball}
\end{figure}

\begin{figure}[!htbp]
\centering
\includegraphics[width=\textwidth]{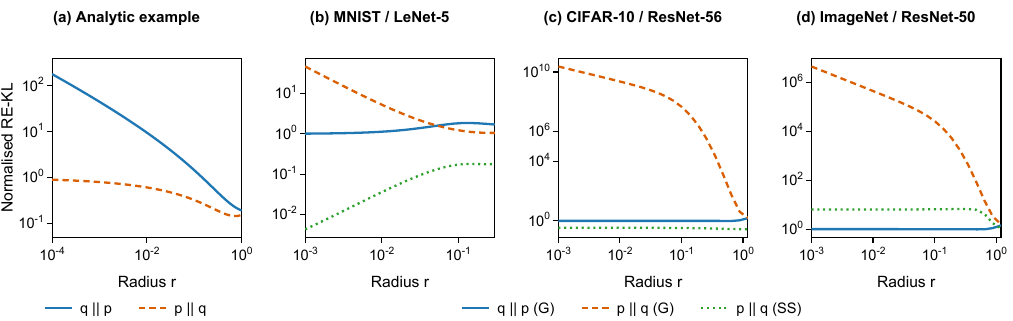}
\caption{Directional normalised local RE-KL. (a) The analytic example at
\(\alpha=3/4\), with uniform \(p\) and cusp \(q\).
(b--d) LeNet-5/MNIST, ResNet-56/CIFAR-10, and ResNet-50/ImageNet,
respectively, with \(p=\pi_0\) and \(\alpha=1/2\). G denotes Gaussian and SS spike-and-slab.
ImageNet curves average three seeds; CIFAR-10 uses one seed.
Vertical scales differ across panels. Panels (b--d) show both Gaussian directions and the
prior-to-spike-and-slab direction.}
\label{fig:nn_directional}
\end{figure}

\paragraph{Local mass across weights and channel gates.}
Figure~\ref{fig:nn_smallball}(a--c) compares the four variational families
on MNIST, CIFAR-10, and ImageNet. Hard-concrete small-ball and RE-KL
results use an effective marginal that retains the zero atom and
replaces the active component by a Gaussian
(Appendix~\ref{app:nn_protocol}). Continuous marginals have vanishing
small-ball mass and \(\MI_{\mathrm{pow}}=1\), whereas positive atoms
produce a nonzero plateau and \(\MI_{\mathrm{pow}}=\infty\).
Figure~\ref{fig:nn_smallball}(d) summarises the reported atom masses:
spike-and-slab gives \(0.52\), \(0.32\), and \(0.065\), while
hard-concrete gives \(0.30\), \(0.03\), and \(0.023\), respectively.
Thus the same distinction appears for individual weights and for channel
gates, including ResNet-50 on ImageNet. These are atom probabilities, not
the realised fractions of pruned parameters or channels.

\paragraph{Directional comparison.}
Figure~\ref{fig:nn_directional}(b--d) compares local RE-KL in the three
networks. For continuous \(q\),
\[
\begin{aligned}
\overline D_\alpha^{B_r}(q\|\pi_0)&\to1
\\
\text{and}\quad\overline D_\alpha^{B_r}(\pi_0\|q)&\to\infty;
\end{aligned}
\]
retaining a positive atom keeps both directions bounded. In the \(\pi_0\|q\) direction, the prior's zero atom contributes explicitly as mass singular with respect to continuous \(q\).
At the smallest reported ImageNet
radius, the prior-to-approximation values are \(4.4\times10^6\),
\(6.3\), and \(24.0\) for Gaussian, spike-and-slab, and the hard-concrete
approximation, respectively. Continuous ResNet gates concentrate near the ``on'' value \(1\),
leaving little mass near zero and producing large local divergence.
The comparisons concern marginal agreement with the prior. Accuracy
relative to the exact posterior is not measured, and a finite-radius
curve does not establish the asymptotic threshold in
Theorem~\ref{thm:one-sided-MI-pow}(i).

\FloatBarrier
\Needspace{10\baselineskip}
\section{Conclusion}
Mass Indices describe how local probability mass changes under Bayesian updating, while local RE-KL gives directional sufficient conditions for preserving or ordering those scales. The experiments illustrate these comparisons from analytic models to neural-network marginals. Global variational accuracy alone does not guarantee local mass fidelity. The neural-network comparisons use finite-radius marginals and a prior reference, so they do not establish asymptotic approximation accuracy relative to the exact posterior. Power indices \(0\) and \(\infty\) do not resolve all finer super-polynomial and sub-polynomial behaviours.

\phantomsection
\pdfbookmark[1]{References}{references}
\bibliographystyle{plain}
\begingroup
\small
\setstretch{1.05}
\bibliography{references}
\endgroup

\clearpage

\appendix

\section{Examples and Counterexamples}
\label{app:examples-counterexamples}

\paragraph{Asymptotic notation.}
For non-negative functions \(f\) and \(g\), with respect to the limit under consideration,
\(f=O(g),\;f=\Omega(g),\; f\asymp g\) denote upper, lower, and two-sided comparison up to positive constants, respectively.

\begin{example}[Oscillating small-ball order]
\label{exp:osc}
This example shows why the upper and lower Power Mass Indices need not
coincide. We construct a probability measure whose small-ball mass behaves
like \(r^a\) along one sequence of radii, but like \(r^b\) along another
sequence of radii, where \(0<a<b\).

For simplicity, take \(\Theta=\mathbb R^d\). The same construction works
whenever \(\Theta\) contains a sufficiently small line segment starting from
\(\pa\). Fix \(0<a<b\), and set
\[
t_n:=\left(\frac{2b}{a}\right)^n,
\qquad n=0,1,2,\ldots .
\]
We define an increasing continuous function \(\phi\) on large positive
values of \(t\) by prescribing its values on the sequence \((t_n)\), and
then interpolating linearly between consecutive points:
\[
\phi(t_{2m})=a t_{2m},
\qquad
\phi(t_{2m+1})=b t_{2m+1}.
\]
The prescribed values are strictly increasing. Indeed,
\[
b t_{2m+1}>a t_{2m},
\qquad
a t_{2m+2}=2b t_{2m+1}.
\]
Hence the piecewise-linear interpolation is increasing.

Now define, for small \(r>0\),
\[
m(r):=\exp\{-\phi(\log(1/r))\}.
\]
As \(r\downarrow0\), we have \(\log(1/r)\to\infty\). Since \(\phi\) is
increasing and tends to infinity, \(m(r)\downarrow0\). Also, as a function
of \(r\), \(m(r)\) is increasing. Thus \(m\) can be used as the distribution
function of radial mass near \(\pa\).

Choose \(\varepsilon>0\) such that \(m(\varepsilon)<1\). By the
Lebesgue--Stieltjes construction, there exists a finite measure \(\nu\) on
\((0,\varepsilon]\) such that
\[
\nu((0,r])=m(r),
\qquad 0<r\le \varepsilon .
\]
Push this measure forward by the map
\[
s\mapsto \pa+s e_1,
\]
where \(e_1\) is the first coordinate vector. Finally, place the remaining
mass \(1-m(\varepsilon)\) outside \(B_\varepsilon(\pa)\). This defines a
probability measure \(p\).

Since \(m\) is continuous, the measure \(\nu\) has no atoms. Therefore no
mass is placed on spheres centred at \(\pa\). Consequently, for every
\(0<r\le \varepsilon\),
\[
p(B_r(\pa))=m(r).
\]

Let
\[
r_n:=e^{-t_n}.
\]
Then \(\log(1/r_n)=t_n\). Along the even subsequence, we obtain
\[
p(B_{r_{2m}}(\pa))
=
\exp\{-\phi(t_{2m})\}
=
\exp\{-a t_{2m}\}
=
r_{2m}^a.
\]
Along the odd subsequence, we obtain
\[
p(B_{r_{2m+1}}(\pa))
=
\exp\{-\phi(t_{2m+1})\}
=
\exp\{-b t_{2m+1}\}
=
r_{2m+1}^b.
\]

Moreover, by construction, the ratio \(\phi(t)/t\) stays between \(a\) and
\(b\) on each interpolation interval. Hence, for all sufficiently small
\(r\),
\[
r^b
\le
p(B_r(\pa))
\le
r^a .
\]
The even subsequence shows that the upper power order cannot be better than
\(a\), while the odd subsequence shows that the lower power order cannot be
better than \(b\). Therefore
\[
\overline{\MI}_{\mathrm{pow}}(p,\pa)=\frac{d}{a},
\qquad
\underline{\MI}_{\mathrm{pow}}(p,\pa)=\frac{d}{b}.
\]
Since \(a<b\), these two values are different. Hence the Power Mass Index at
\(\pa\) does not exist.
\end{example}

\begin{example}[Directionality of local RE--KL conditions]
\label{exp:dir}
This example illustrates that the local RE--KL conditions are directional.
We work at the point \(0\). Let \(p\) and \(q\) be probability measures on
\([-1,1]\) with densities
\[
\frac{\mathrm d p}{\mathrm d x}(x)=\frac12,
\qquad
\frac{\mathrm d q}{\mathrm d x}(x)=\frac14 |x|^{-1/2}.
\]
Both are probability measures, since
\[
\int_{-1}^{1} \frac12 \,\mathrm dx =1,
\qquad
\int_{-1}^{1} \frac14 |x|^{-1/2}\,\mathrm dx =1.
\]
For \(0<r<1\), writing \(B_r=B_r(0)=(-r,r)\), we have
\[
p(B_r)
=
\int_{-r}^{r}\frac12\,\mathrm dx
=
r,
\]
whereas
\[
q(B_r)
=
\int_{-r}^{r}\frac14 |x|^{-1/2}\,\mathrm dx
=
r^{1/2}.
\]
Therefore
\[
\MI_{\mathrm{pow}}(p,0)=1,
\qquad
\MI_{\mathrm{pow}}(q,0)=2.
\]
Thus \(q\) has more local mass near \(0\) than \(p\), in the power-order
sense.

We first consider the direction \(p\|q\). Since
\[
\frac{\mathrm d p}{\mathrm d q}(x)
=
2|x|^{1/2},
\]
the likelihood ratio tends to \(0\) as \(x\to0\). Moreover, on \(B_r\),
\[
0\le \frac{\mathrm d p}{\mathrm d q}(x)\le 2r^{1/2}.
\]
Hence the ratio \(\mathrm d p/\mathrm d q\) converges uniformly to \(0\)
on \(B_r\) as \(r\downarrow0\). By continuity of \(f_\alpha\) at \(0\),
\[
f_\alpha\!\left(\frac{\mathrm d p}{\mathrm d q}(x)\right)
\longrightarrow
f_\alpha(0)
\qquad
\text{uniformly on }B_r.
\]
Consequently,
\[
\overline D_\alpha^{B_r}(p\|q)
:=
\frac{1}{q(B_r)}
\int_{B_r}
f_\alpha\!\left(\frac{\mathrm d p}{\mathrm d q}\right)
\,\mathrm d q
\longrightarrow
f_\alpha(0)
=
1.
\]
Thus the local RE--KL condition in the direction \(p\|q\) is satisfied.
The resulting one-sided conclusion is exactly the expected one:
\[
\MI_{\mathrm{pow}}(q,0)
\ge
\MI_{\mathrm{pow}}(p,0).
\]

The reverse direction behaves differently. In the direction \(q\|p\), we have
\[
\frac{\mathrm d q}{\mathrm d p}(x)
=
\frac12 |x|^{-1/2},
\]
which diverges as \(x\to0\). For large \(t\),
\[
f_\alpha(t)
\sim
\frac{\alpha}{1-\alpha}t.
\]
Therefore, near \(0\),
\[
f_\alpha\!\left(\frac{\mathrm d q}{\mathrm d p}(x)\right)
\asymp
\frac{\mathrm d q}{\mathrm d p}(x).
\]
It follows that
\[
\int_{B_r}
f_\alpha\!\left(\frac{\mathrm d q}{\mathrm d p}\right)
\,\mathrm d p
\asymp
\int_{B_r}
\frac{\mathrm d q}{\mathrm d p}
\,\mathrm d p
=
q(B_r).
\]
After normalisation by \(p(B_r)\), we obtain
\[
\overline D_\alpha^{B_r}(q\|p)
:=
\frac{1}{p(B_r)}
\int_{B_r}
f_\alpha\!\left(\frac{\mathrm d q}{\mathrm d p}\right)
\,\mathrm d p
\asymp
\frac{q(B_r)}{p(B_r)}
=
\frac{r^{1/2}}{r}
=
r^{-1/2}
\longrightarrow
\infty.
\]
Hence the corresponding \(q\|p\) local RE--KL condition fails. This shows
that the two directions are not interchangeable: \(p\|q\) controls a
different local comparison from \(q\|p\).
\end{example}

\paragraph{Why the singular term matters.}
Write \(B_r=B_r(\pa)\). Let \(p\) have a continuous, positive Lebesgue density near \(\pa\), and let
\(q=(1-\varepsilon)p+\varepsilon\delta_{\pa}\), \(0<\varepsilon<1\).
Then \(q_{\perp p}=\varepsilon\delta_{\pa}\), \(p(B_r)\asymp r^d\), and
\(q(B_r)\to\varepsilon\), so their Power Mass Indices are \(1\) and \(\infty\), respectively.
The normalised absolutely continuous contribution is the constant
\(f_\alpha(1-\varepsilon)\), which can be arbitrarily small. The full divergence instead satisfies
\[
\overline D_\alpha^{B_r}(q\|p)
=f_\alpha(1-\varepsilon)+\frac{\alpha}{1-\alpha}\frac{\varepsilon}{p(B_r)}
\longrightarrow\infty.
\]
Thus retaining the singular term detects the atom that dominates the shrinking-ball mass.

\section{Proofs}
\label{app:proofs}

\subsection{Additional lemmas}

\begin{lemma}[Gauge identification]
\label{lem:gauge-identification}
Let \(m(r)\ge0\) be a small-ball mass function and let \(\ell(r):=-\log r\).

\textnormal{(i)} Assume \(a>0\). If, for every \(\varepsilon>0\),
\[
 m(r)=O(r^{a-\varepsilon}),
 \qquad
 m(r)=\Omega(r^{a+\varepsilon}),
 \qquad r\downarrow0,
\]
then the Power Mass Index exists and equals \(d/a\).

\textnormal{(ii)} Assume that the power exponent is \(a>0\). If, for some
\(b\in\mathbb R\) and every \(\varepsilon>0\),
\[
 m(r)=O\!\left(r^a\ell(r)^{b+\varepsilon}\right),
 \qquad
 m(r)=\Omega\!\left(r^a\ell(r)^{b-\varepsilon}\right),
 \qquad r\downarrow0,
\]
then the Logarithmic Mass Index exists and equals \(b\).
\end{lemma}

\begin{proof}
For \textnormal{(i)}, the upper bounds imply that every exponent strictly below
\(a\) is an upper power gauge, while the lower bounds imply that every exponent
strictly above \(a\) is a lower power gauge. Therefore the sharp threshold in
exponent form is \(a\). Since Definition~\ref{define:MI} parametrises power
bounds as \(r^{1/\eta}\), this gives
\[
 \inf\mathcal U_{\mathrm{pow}}=\sup\mathcal L_{\mathrm{pow}}=\frac1a,
\]
and hence \(\MI_{\mathrm{pow}}=d/a\).

For \textnormal{(ii)}, the definition of the logarithmic gauge is already taken
at the identified power scale \(a\). The upper bounds imply that every
logarithmic exponent strictly above \(b\) is an upper log gauge, and the lower
bounds imply that every logarithmic exponent strictly below \(b\) is a lower log
gauge. Thus
\[
 \inf\mathcal U_{\mathrm{log}}=\sup\mathcal L_{\mathrm{log}}=b,
\]
so \(\MI_{\mathrm{log}}=b\).
\end{proof}

\begin{lemma}[Radial power transform has bounded KL]
\label{lem:radial-power-kl}
Let \(\mu\) be a probability measure on \([0,\rho)\), and write
\(F(r):=\mu([0,r))\). Fix \(t>0\), and let \(\nu\) be the probability measure
on \([0,\rho)\) determined by
\[
 \nu([0,r))=F(r)^t,
 \qquad 0<r<\rho .
\]
Then \(\nu\ll\mu\), and
\[
 D_{\mathrm{KL}}(\nu\|\mu)
 \le
 \log t+\frac1t-1 .
\]
\end{lemma}

\begin{proof}
Let \(\lambda\) be Lebesgue measure on \([0,1]\), and let \(\lambda_t\) have
density
\[
 \frac{\mathrm d\lambda_t}{\mathrm d\lambda}(u)=t u^{t-1},
 \qquad 0<u<1 .
\]
Let \(Q\) be a quantile map that pushes \(\lambda\) forward to \(\mu\). Then the
measure defined above is the pushforward \(Q_\#\lambda_t\). Since
\(\lambda_t\ll\lambda\), we have \(\nu\ll\mu\). By data processing for KL
divergence,
\[
 D_{\mathrm{KL}}(\nu\|\mu)
 \le
 D_{\mathrm{KL}}(\lambda_t\|\lambda).
\]
The right-hand side is explicit:
\[
 D_{\mathrm{KL}}(\lambda_t\|\lambda)
 =\int_0^1t u^{t-1}\log(tu^{t-1})\,\mathrm d u
 =\log t+\frac1t-1 .
\]
This proves the claim.
\end{proof}

\subsection*{Proof of Theorem~\ref{prop:basic-calibrations}}
\begin{proof}[Proof of Theorem~\ref{prop:basic-calibrations}]
Let
\[
 m(r):=p(B_r(\pa)),
 \qquad
 \ell(r):=-\log r .
\]

\textnormal{(1) Local hole.}
If \(p(B_{r_0}(\pa))=0\), then \(m(r)=0\) for all \(0<r<r_0\). Hence
\(m(r)=O(r^{1/\eta})\) for every \(\eta>0\), so
\(\inf\mathcal U_{\mathrm{pow}}=0\). On the other hand,
\(m(r)=\Omega(r^{1/\eta})\) holds for no \(\eta>0\). By the convention
\(\sup\emptyset=0\), we also have \(\sup\mathcal L_{\mathrm{pow}}=0\). Thus
\[
 \MI_{\mathrm{pow}}(p,\pa)=0 .
\]

\textnormal{(2) Regular continuous mass.}
Suppose that \(p\) admits a continuous density \(\rho_p\) on \(B_{r_0}(\pa)\)
and \(\rho_p(\pa)>0\). Then, for some \(r_1\in(0,r_0)\) and constants
\(0<c<C<\infty\),
\[
 c\le \rho_p(x)\le C,
 \qquad x\in B_{r_1}(\pa).
\]
Therefore, for \(0<r<r_1\),
\[
 c\,\mathrm{Vol}(B_r(\pa))
 \le
 p(B_r(\pa))
 \le
 C\,\mathrm{Vol}(B_r(\pa)).
\]
Since \(\mathrm{Vol}(B_r(\pa))\asymp r^d\), we get
\[
 p(B_r(\pa))\asymp r^d .
\]
The definition of the Power Mass Index then gives
\[
 \MI_{\mathrm{pow}}(p,\pa)=1 .
\]

\textnormal{(3) Power-log asymptotics.}
Assume
\[
 p(B_r(\pa))\asymp r^a \ell(r)^b,
 \qquad r\downarrow0 .
\]
For every \(\varepsilon>0\), logarithmic factors are slower than powers, hence
\[
 r^a\ell(r)^b=O(r^{a-\varepsilon}),
 \qquad
 r^a\ell(r)^b=\Omega(r^{a+\varepsilon}).
\]
Lemma~\ref{lem:gauge-identification}\textnormal{(i)} gives
\[
 \MI_{\mathrm{pow}}(p,\pa)=\frac da .
\]
At the corresponding power scale \(a_{p,\pa}=a\), the two-sided asymptotic
also gives, for every \(\varepsilon>0\),
\[
 p(B_r(\pa))=O\!\left(r^a\ell(r)^{b+\varepsilon}\right),
 \qquad
 p(B_r(\pa))=\Omega\!\left(r^a\ell(r)^{b-\varepsilon}\right).
\]
Lemma~\ref{lem:gauge-identification}\textnormal{(ii)} gives
\[
 \MI_{\mathrm{log}}(p,\pa)=b .
\]

\textnormal{(4) Atomic mass.}
If \(p(\{\pa\})>0\), then \(m(r)\ge p(\{\pa\})>0\) for every \(r>0\). Hence
\(m(r)=O(r^{1/\eta})\) holds for no \(\eta>0\), so
\(\inf\mathcal U_{\mathrm{pow}}=\infty\). Also
\(m(r)=\Omega(r^{1/\eta})\) holds for every \(\eta>0\), so
\(\sup\mathcal L_{\mathrm{pow}}=\infty\). Therefore
\[
 \MI_{\mathrm{pow}}(p,\pa)=\infty .
\]
The proof is complete.
\end{proof}

\subsection*{Proof of the RE-KL relations}
\begin{proof}[Proof of the RE-KL relations in Section~\ref{subsec:rekl}]
Let \(\nu=p+q\), \(u=\mathrm dq/\mathrm d\nu\), and \(v=\mathrm dp/\mathrm d\nu\). Separating \(\{v>0\}\) from the singular set \(\{v=0\}\), the definition gives
\[
 D_\alpha^\Theta(q\|p)
 =\frac{\alpha\int_\Theta u\,\mathrm d\nu+(1-\alpha)\int_\Theta v\,\mathrm d\nu
       -\int_\Theta u^\alpha v^{1-\alpha}\,\mathrm d\nu}{1-\alpha}
 =\frac{1-\int_\Theta u^\alpha v^{1-\alpha}\,\mathrm d\nu}{1-\alpha}.
\]
This is the Tsallis-\(\alpha\) form for arbitrary probability measures. At \(\alpha=1/2\), it gives
\[
 D_{1/2}^\Theta(q\|p)
 =2\left(1-\int_\Theta\sqrt{uv}\,\mathrm d\nu\right)
 =2\mathcal H^2(q,p),
\]
under the Hellinger convention used in the paper.

For the monotone KL limit, first assume \(q\ll p\) and write \(R:=\mathrm dq/\mathrm dp\). For \(x>0\), differentiating
\[
 f_\alpha(x)=\frac{\alpha x-x^\alpha+1-\alpha}{1-\alpha}
\]
with respect to \(\alpha\) gives
\[
 \partial_\alpha f_\alpha(x)
 =\frac{x-x^\alpha\{1+(1-\alpha)\log x\}}{(1-\alpha)^2}
 =\frac{x^\alpha}{(1-\alpha)^2}
 \left(x^{1-\alpha}-1-\log x^{1-\alpha}\right).
\]
The elementary inequality \(u-1-\log u\ge0\), \(u>0\), therefore implies
\(\partial_\alpha f_\alpha(x)\ge0\). At \(x=0\), \(f_\alpha(0)=1\) for all
\(\alpha\in(0,1)\), so monotonicity is trivial. Moreover, for each fixed
\(x\ge0\), a Taylor expansion of \(x^\alpha\) at \(\alpha=1\), with the endpoint
values interpreted by continuity, gives
\[
 f_\alpha(x)\uparrow x\log x-x+1,
 \qquad \alpha\uparrow1 .
\]
By monotone convergence,
\[
 \lim_{\alpha\uparrow1}D_\alpha^\Theta(q\|p)
 =\int_\Theta (R\log R-R+1)\,\mathrm d p
 =\int_\Theta R\log R\,\mathrm d p
 =D_{\mathrm{KL}}(q\|p),
\]
because \(\int R\,\mathrm d p=1\). If \(q\not\ll p\), then the singular part is
charged by
\[
 \frac{\alpha}{1-\alpha}q_{\perp p}(\Theta),
\]
which diverges as \(\alpha\uparrow1\), matching
\(D_{\mathrm{KL}}(q\|p)=\infty\). This proves the stated relations.
\end{proof}
\subsection*{Proof of Theorem~\ref{thm:kl-mi-discontinuous}}
\begin{proof}[Proof of Theorem~\ref{thm:kl-mi-discontinuous}]
Let
\[
 m:=\MI_{\mathrm{pow}}(p,\pa)\in(0,\infty).
\]
Since the Power Mass Index of \(p\) at \(\pa\) exists, Definition~\ref{define:MI} gives
\[
 \inf\mathcal U_{\mathrm{pow}}^p
 =
 \sup\mathcal L_{\mathrm{pow}}^p
 =
 \frac md,
\]
where the superscript indicates that the corresponding sets are defined with
respect to \(p\).

Fix \(\kappa\in(0,\infty)\) with \(\kappa\neq m\), and set
\[
 t:=\frac m\kappa>0 .
\]
Choose a sequence \(\rho_n\downarrow0\), and define
\[
 E_n:=B_{\rho_n}(\pa),
 \qquad
 a_n:=p(E_n).
\]
Because \(m<\infty\), the measure \(p\) has no atom at \(\pa\), and
\(a_n\to0\). Since \(m>0\), there is no local hole around \(\pa\), so
\(a_n>0\) for all sufficiently large \(n\). Passing to this tail of the
sequence, define the conditional probability measure
\[
 p_n(A):=\frac{p(A\cap E_n)}{a_n}.
\]

Let
\[
 R(x):=\|x-\pa\|,
 \qquad
 F_n(r):=p_n(B_r(\pa)),
 \qquad 0<r<\rho_n .
\]
Let \(\mu_n\) be the radial distribution of \(R\) under \(p_n\). To keep the
open-ball convention used in the Mass Index definition, we use half-open radial
intervals and write
\[
 \mu_n([0,r))=F_n(r)=p_n(B_r(\pa)).
\]
This convention avoids any ambiguity from possible mass on spheres. Define a
new radial probability measure \(\nu_n\) by
\[
 \nu_n([0,r))=F_n(r)^t,
 \qquad 0<r<\rho_n .
\]
Disintegrate \(p_n\) with respect to the radial variable:
\[
 p_n(\mathrm d x)=\mu_n(\mathrm d s)K_n(s,\mathrm d\omega),
 \qquad s=\|x-\pa\|,
\]
where \(K_n(s,\cdot)\) is a regular conditional angular distribution. Define
\[
 \eta_n(\mathrm d x):=\nu_n(\mathrm d s)K_n(s,\mathrm d\omega).
\]
Then \(\eta_n\ll p_n\), and, for \(0<r<\rho_n\),
\[
 \eta_n(B_r(\pa))
 =\nu_n([0,r))
 =F_n(r)^t
 =\left(\frac{p(B_r(\pa))}{a_n}\right)^t .
\]

By Lemma~\ref{lem:radial-power-kl},
\[
 D_{\mathrm{KL}}(\nu_n\|\mu_n)
 \le
 C_t:=\log t+\frac1t-1<\infty .
\]
Since \(\eta_n\) and \(p_n\) have the same conditional angular distributions, the
Radon--Nikodym derivative is
\[
 \frac{\mathrm d\eta_n}{\mathrm d p_n}(x)
 =
 \frac{\mathrm d\nu_n}{\mathrm d\mu_n}(\|x-\pa\|)
 \quad p_n\text{-a.e.}
\]
Therefore
\[
\begin{aligned}
 D_{\mathrm{KL}}(\eta_n\|p_n)
 &=\int \log\!\left(\frac{\mathrm d\eta_n}{\mathrm d p_n}\right)
       \,\mathrm d\eta_n  \\
 &=\int \log\!\left(\frac{\mathrm d\nu_n}{\mathrm d\mu_n}\right)
       \,\mathrm d\nu_n
 =D_{\mathrm{KL}}(\nu_n\|\mu_n)
 \le C_t .
\end{aligned}
\]
This bound is uniform in \(n\).

Now define
\[
 q_n(A):=p(A\cap E_n^c)+a_n\eta_n(A\cap E_n).
\]
Then \(q_n\) is a probability measure, \(q_n\ll p\), and \(q_n=p\) on
\(E_n^c\). For \(0<r<\rho_n\),
\[
 q_n(B_r(\pa))
 =a_n\eta_n(B_r(\pa))
 =a_n^{1-t}p(B_r(\pa))^t .
\]
The positive constant \(a_n^{1-t}\) does not affect the local \(O\)- or
\(\Omega\)-classes as \(r\downarrow0\). Therefore, for every \(\eta>0\),
\[
 q_n(B_r(\pa))=O(r^{1/\eta})
 \quad\Longleftrightarrow\quad
 p(B_r(\pa))=O(r^{1/(t\eta)}),
\]
and similarly
\[
 q_n(B_r(\pa))=\Omega(r^{1/\eta})
 \quad\Longleftrightarrow\quad
 p(B_r(\pa))=\Omega(r^{1/(t\eta)}).
\]
Thus
\[
 \inf\mathcal U_{\mathrm{pow}}^{q_n}
 =\frac1t\inf\mathcal U_{\mathrm{pow}}^p
 =\frac1t\frac md
 =\frac \kappa d,
\]
and
\[
 \sup\mathcal L_{\mathrm{pow}}^{q_n}
 =\frac1t\sup\mathcal L_{\mathrm{pow}}^p
 =\frac1t\frac md
 =\frac \kappa d.
\]
Consequently,
\[
 \MI_{\mathrm{pow}}(q_n,\pa)=\kappa
\]
for every \(n\).

Finally, since \(q_n=p\) outside \(E_n\) and the conditional law inside \(E_n\)
is changed from \(p_n\) to \(\eta_n\),
\[
 D_{\mathrm{KL}}(q_n\|p)
 =a_nD_{\mathrm{KL}}(\eta_n\|p_n)
 \le a_nC_t.
\]
As \(a_n=p(B_{\rho_n}(\pa))\to0\), we obtain
\[
 D_{\mathrm{KL}}(q_n\|p)\to0 .
\]
Thus \(q_n\in\mathcal M_{\pa}\), \(D_{\mathrm{KL}}(q_n\|p)\to0\), and
\(\MI_{\mathrm{pow}}(q_n,\pa)=\kappa\) for every \(n\). This proves the theorem.
\end{proof}

\subsection*{Proof of Theorem~\ref{thm:local-likelihood-scaling}}
\begin{proof}[Proof of Theorem~\ref{thm:local-likelihood-scaling}]
Write
\[
 m_0(r):=\pi_0(B_r(\pa)),
 \qquad
 \ell(r):=\log\frac1r .
\]
First assume that
\[
 0<\MI_{\mathrm{pow}}(\pi_0,\pa)<\infty,
 \qquad
 a:=d\,\MI_{\mathrm{pow}}^{-1}(\pi_0,\pa).
\]
The condition in the theorem is exactly
\[
 a+\gamma>0 .
\]
Let
\[
 I(r):=
 \int_{B_r(\pa)}
 \|\vartheta-\pa\|^\gamma
 \left(\log\frac1{\|\vartheta-\pa\|}\right)^\beta
 \pi_0(\mathrm d\vartheta).
\]
Since \(0<\MI_{\mathrm{pow}}(\pi_0,\pa)<\infty\), the prior has no atom at
\(\pa\), so the value of the integrand at \(\pa\) is irrelevant. The local
likelihood assumption gives constants \(0<c<C<\infty\) and \(r_0>0\) such that,
for \(\vartheta\in B_{r_0}(\pa)\setminus\{\pa\}\),
\[
 c\|\vartheta-\pa\|^\gamma
 \left(\log\frac1{\|\vartheta-\pa\|}\right)^\beta
 \le L(\vartheta\mid\mathcal D)
 \le C\|\vartheta-\pa\|^\gamma
 \left(\log\frac1{\|\vartheta-\pa\|}\right)^\beta .
\]
By Bayes' formula,
\[
 \pi_{\mathcal D}(B_r(\pa))
 =Z_{\mathcal D}^{-1}
 \int_{B_r(\pa)}L(\vartheta\mid\mathcal D)\,\pi_0(\mathrm d\vartheta).
\]
Since \(Z_{\mathcal D}\in(0,\infty)\), the normalising constant is fixed, and
therefore
\[
 \pi_{\mathcal D}(B_r(\pa))\asymp I(r),
 \qquad r\downarrow0 .
\]
It remains to identify the power and logarithmic orders of \(I(r)\).

For the power order, the definition of the Power Mass Index gives, for every
sufficiently small \(u>0\),
\[
 m_0(r)=O(r^{a-u}),
 \qquad
 m_0(r)=\Omega(r^{a+u}).
\]
We claim that, for every \(\varepsilon>0\),
\[
 I(r)=O(r^{a+\gamma-\varepsilon}),
 \qquad
 I(r)=\Omega(r^{a+\gamma+\varepsilon}).
\]
For the upper bound, choose \(u,v>0\) such that
\[
 u+v<\varepsilon,
 \qquad
 a+\gamma-u-v>0 .
\]
Since logarithmic factors are slower than powers,
\[
 \left(\log\frac1t\right)^\beta\le C_vt^{-v}
\]
for all sufficiently small \(t>0\), with \(C_v\) independent of the annulus
index used below. Let \(r_j=e^{-j}r\), and set
\[
 A_j:=B_{r_j}(\pa)\setminus B_{r_{j+1}}(\pa),
 \qquad j\ge0 .
\]
Then \(B_r(\pa)\setminus\{\pa\}=\bigcup_{j\ge0}A_j\). On \(A_j\),
\(r_{j+1}\le \|\vartheta-\pa\|<r_j\). Hence, after absorbing the fixed factor
\(e^{|\gamma-v|}\) when \(\gamma-v<0\),
\[
 \|\vartheta-\pa\|^\gamma
 \left(\log\frac1{\|\vartheta-\pa\|}\right)^\beta
 \le C r_j^{\gamma-v},
\]
where the constant is uniform in \(j\) and in all sufficiently small \(r\).
Therefore
\[
\begin{aligned}
 I(r)
 &\le C\sum_{j=0}^{\infty} r_j^{\gamma-v}\pi_0(A_j) \\
 &\le C\sum_{j=0}^{\infty} r_j^{\gamma-v}m_0(r_j) \\
 &\le C\sum_{j=0}^{\infty} r_j^{a+\gamma-u-v}
 \le C r^{a+\gamma-u-v}.
\end{aligned}
\]
This proves \(I(r)=O(r^{a+\gamma-\varepsilon})\).

For the lower bound, choose \(u,v>0\) and \(s>1\), with \(s\) sufficiently close
to one, such that
\[
 s(a-u)>a+u,
 \qquad
 u+(s-1)\max\{\gamma,0\}+v<\varepsilon .
\]
Let
\[
 A(r):=B_r(\pa)\setminus B_{r^s}(\pa).
\]
The prior power bounds imply
\[
 \pi_0(A(r))
 =m_0(r)-m_0(r^s)
 \ge c r^{a+u}-C r^{s(a-u)}
 \ge c r^{a+u}
\]
for all sufficiently small \(r\). On \(A(r)\),
\[
 \|\vartheta-\pa\|^\gamma
 \ge C r^{\gamma+(s-1)\max\{\gamma,0\}},
 \qquad
 \left(\log\frac1{\|\vartheta-\pa\|}\right)^\beta
 \ge C r^v .
\]
Consequently,
\[
 I(r)
 \ge C r^{a+\gamma+u+(s-1)\max\{\gamma,0\}+v} .
\]
Because
\[
 u+(s-1)\max\{\gamma,0\}+v<\varepsilon,
\]
we have
\[
 r^{a+\gamma+u+(s-1)\max\{\gamma,0\}+v}
 =\Omega(r^{a+\gamma+\varepsilon}),
\]
and hence \(I(r)=\Omega(r^{a+\gamma+\varepsilon})\).
The preceding two estimates satisfy Lemma~\ref{lem:gauge-identification}\textnormal{(i)}
with exponent \(a+\gamma\). Therefore
\[
 d\,\MI_{\mathrm{pow}}^{-1}(\pi_{\mathcal D},\pa)
 =a+\gamma
 =d\,\MI_{\mathrm{pow}}^{-1}(\pi_0,\pa)+\gamma .
\]
Equivalently,
\[
 \MI_{\mathrm{pow}}^{-1}(\pi_{\mathcal D},\pa)
 =\MI_{\mathrm{pow}}^{-1}(\pi_0,\pa)+d^{-1}\gamma .
\]

Now assume that \(\MI_{\mathrm{log}}(\pi_0,\pa)\) exists and is finite, and set
\[
 b:=\MI_{\mathrm{log}}(\pi_0,\pa),
 \qquad
 \lambda:=a+\gamma>0 .
\]
Then, for every sufficiently small \(u>0\),
\[
 m_0(r)=O\!\left(r^a\ell(r)^{b+u}\right),
 \qquad
 m_0(r)=\Omega\!\left(r^a\ell(r)^{b-u}\right).
\]
We show that, for every \(\varepsilon>0\),
\[
 I(r)=O\!\left(r^\lambda\ell(r)^{b+\beta+\varepsilon}\right),
 \qquad
 I(r)=\Omega\!\left(r^\lambda\ell(r)^{b+\beta-\varepsilon}\right).
\]
For the upper bound, use the same annuli \(A_j\), and choose
\(u\in(0,\varepsilon)\). On \(A_j\), the relations
\(r_{j+1}\le \|\vartheta-\pa\|<r_j\) and
\(\log(1/\|\vartheta-\pa\|)\asymp\ell(r_j)\) give
\[
 \|\vartheta-\pa\|^\gamma
 \left(\log\frac1{\|\vartheta-\pa\|}\right)^\beta
 \le C r_j^\gamma\ell(r_j)^\beta ,
\]
with a constant uniform in \(j\) and in all sufficiently small \(r\).
Hence
\[
\begin{aligned}
 I(r)
 &\le C\sum_{j=0}^{\infty}r_j^\gamma\ell(r_j)^\beta m_0(r_j)\\
 &\le C\sum_{j=0}^{\infty}r_j^{a+\gamma}\ell(r_j)^{b+\beta+u}\\
 &=C r^\lambda\sum_{j=0}^{\infty}e^{-j\lambda}(\ell(r)+j)^{b+\beta+u}.
\end{aligned}
\]
For \(c=b+\beta+u\), the elementary geometric estimate
\[
 \sum_{j=0}^{\infty}e^{-j\lambda}(\ell(r)+j)^c
 \le C\ell(r)^c
\]
holds uniformly for all sufficiently small \(r\). Indeed, if \(c<0\), then
\((\ell(r)+j)^c\le \ell(r)^c\); if \(c\ge0\), then
\((\ell(r)+j)^c\le \ell(r)^c(1+j)^c\), and
\(\sum_{j\ge0}e^{-j\lambda}(1+j)^c<\infty\). Therefore
\[
 I(r)\le C r^\lambda\ell(r)^{b+\beta+u}.
\]
This gives the desired upper logarithmic bound.

For the lower bound, write \(\gamma_+:=\max\{\gamma,0\}\). Choose \(u>0\) and
\(K>0\) such that
\[
 Ka>2u,
 \qquad
 u+K\gamma_+<\varepsilon .
\]
Set
\[
 \rho(r):=r\ell(r)^{-K},
 \qquad
 A(r):=B_r(\pa)\setminus B_{\rho(r)}(\pa).
\]
Using the lower and upper logarithmic bounds for \(m_0\),
\[
\begin{aligned}
 \pi_0(A(r))
 &=m_0(r)-m_0(\rho(r))\\
 &\ge c r^a\ell(r)^{b-u}
 -C r^a\ell(r)^{b+u-Ka}
 \ge c r^a\ell(r)^{b-u}
\end{aligned}
\]
for all sufficiently small \(r\). On \(A(r)\),
\[
 \|\vartheta-\pa\|^\gamma\ge C r^\gamma\ell(r)^{-K\gamma_+},
 \qquad
 \log\frac1{\|\vartheta-\pa\|}\asymp\ell(r).
\]
Therefore
\[
 I(r)
 \ge C r^{a+\gamma}\ell(r)^{b+\beta-u-K\gamma_+}
 =\Omega\!\left(r^\lambda\ell(r)^{b+\beta-\varepsilon}\right).
\]
The upper and lower logarithmic bounds satisfy
Lemma~\ref{lem:gauge-identification}\textnormal{(ii)} at the already identified power scale
\(\lambda\). Hence
\[
 \MI_{\mathrm{log}}(\pi_{\mathcal D},\pa)
 =b+\beta
 =\MI_{\mathrm{log}}(\pi_0,\pa)+\beta .
\]

It remains to treat the case \(\MI_{\mathrm{pow}}(\pi_0,\pa)=0\). In this case,
for every \(M>0\),
\[
 m_0(r)=O(r^M),
 \qquad r\downarrow0 .
\]
Fix any \(N>0\). Choose \(v>0\), and then choose \(M\) so large that
\[
 M+\gamma-v>N,
 \qquad
 M+\gamma-v>0 .
\]
Using the same annular decomposition as above and the bound
\(m_0(r_j)=O(r_j^M)\), we obtain
\[
\begin{aligned}
 I(r)
 &\le C\sum_{j=0}^{\infty} r_j^{\gamma-v}m_0(r_j) \\
 &\le C\sum_{j=0}^{\infty} r_j^{M+\gamma-v}
 \le C r^{M+\gamma-v}
 =O(r^N).
\end{aligned}
\]
Thus \(\pi_{\mathcal D}(B_r(\pa))=O(r^N)\) for every \(N>0\). Consequently the
posterior small-ball mass is faster than every polynomial power. Hence every
positive power is an upper gauge and no positive power is a lower gauge, which
gives
\[
 \MI_{\mathrm{pow}}(\pi_{\mathcal D},\pa)=0 .
\]
The proof is complete.
\end{proof}

\subsection*{Proof of Theorem~\ref{thm:acceptance}}
\begin{proof}[Proof of Theorem~\ref{thm:acceptance}]
Set
\[
 m_0(r):=\pi_0(B_r(\pa)),
 \qquad
 m_{\mathcal D}(r):=\pi_{\mathcal D}(B_r(\pa)),
 \qquad
 \ell(r):=-\log r .
\]
By the theorem hypothesis, \(0<\MI_{\mathrm{pow}}(\pi_0,\pa)<\infty\). Set
\[
 a:=d\,\MI_{\mathrm{pow}}^{-1}(\pi_0,\pa).
\]
By the local bounds on \(g\), there exist constants
\(0<c_g<C_g<\infty\) and \(r_0>0\) such that
\[
 c_g\le g(\vartheta)\le C_g,
 \qquad \vartheta\in B_{r_0}(\pa).
\]
By Bayes' formula and the finite marginal assumption, the normalising constant
only changes small-ball masses by a fixed multiplicative constant. Hence, for
\(0<r<r_0\),
\[
 m_{\mathcal D}(r)
 \asymp
 \pi_0(B_r(\pa)\cap S_{\mathcal D})
 =A_{S_{\mathcal D},\pa}(r)m_0(r).
\]
Using the assumed acceptance-ratio scaling,
\begin{equation}\label{eq:acceptance-proof-comparison-new}
 m_{\mathcal D}(r)
 \asymp
 r^\xi\ell(r)^\lambda m_0(r),
 \qquad r\downarrow0 .
\end{equation}
Because \(A_{S_{\mathcal D},\pa}(r)\in[0,1]\), the regime stated in the theorem is the
non-degenerate regime compatible with this two-sided comparison.

For every \(\varepsilon>0\), the prior Power Mass Index gives
\[
 m_0(r)=O(r^{a-\varepsilon}),
 \qquad
 m_0(r)=\Omega(r^{a+\varepsilon}).
\]
Since logarithmic factors are sub-polynomial, for every \(\eta>0\),
\[
 \ell(r)^\lambda=O(r^{-\eta}),
 \qquad
 \ell(r)^\lambda=\Omega(r^{\eta}) .
\]
Given \(\varepsilon>0\), apply the prior bounds with \(\varepsilon/2\) and the
last display with \(\eta=\varepsilon/2\). Then
\eqref{eq:acceptance-proof-comparison-new} gives
\[
 m_{\mathcal D}(r)=O(r^{a+\xi-\varepsilon}),
 \qquad
 m_{\mathcal D}(r)=\Omega(r^{a+\xi+\varepsilon}).
\]
By Lemma~\ref{lem:gauge-identification}\textnormal{(i)}, the posterior power
exponent is \(a+\xi\), and
\[
 d\,\MI_{\mathrm{pow}}^{-1}(\pi_{\mathcal D},\pa)
 =a+\xi
 =d\,\MI_{\mathrm{pow}}^{-1}(\pi_0,\pa)+\xi .
\]
Equivalently,
\[
 \MI_{\mathrm{pow}}^{-1}(\pi_{\mathcal D},\pa)
 =\MI_{\mathrm{pow}}^{-1}(\pi_0,\pa)+d^{-1}\xi .
\]

Now assume that \(\MI_{\mathrm{log}}(\pi_0,\pa)\) exists and is finite, and set
\[
 b:=\MI_{\mathrm{log}}(\pi_0,\pa).
\]
Then, for every \(\varepsilon>0\),
\[
 m_0(r)=O\!\left(r^a\ell(r)^{b+\varepsilon}\right),
 \qquad
 m_0(r)=\Omega\!\left(r^a\ell(r)^{b-\varepsilon}\right).
\]
Combining these bounds with \eqref{eq:acceptance-proof-comparison-new} gives
\[
 m_{\mathcal D}(r)
 =O\!\left(r^{a+\xi}\ell(r)^{b+\lambda+\varepsilon}\right),
 \qquad
 m_{\mathcal D}(r)
 =\Omega\!\left(r^{a+\xi}\ell(r)^{b+\lambda-\varepsilon}\right).
\]
The posterior power exponent has already been identified as \(a+\xi\). Hence
Lemma~\ref{lem:gauge-identification}\textnormal{(ii)} gives
\[
 \MI_{\mathrm{log}}(\pi_{\mathcal D},\pa)
 =b+\lambda
 =\MI_{\mathrm{log}}(\pi_0,\pa)+\lambda .
\]
The proof is complete.
\end{proof}

\subsection*{Proof of Corollary~\ref{cor:softening-noncommuting}}
\begin{proof}[Proof of Corollary~\ref{cor:softening-noncommuting}]
For each \(\tau>0\), define
\[
 Z_\tau:=\int s_\tau(\vartheta)\,\pi_0(\mathrm d\vartheta),
 \qquad
 Z_S:=\int \mathbf 1_{S_{\mathcal D}}(\vartheta)\,
 \pi_0(\mathrm d\vartheta).
\]
For each \(\tau>0\), the softened normalising constant satisfies
\(0<Z_\tau\le C\). By Assumption~\ref{as:bayes} applied to the hard likelihood
\(L=\mathbf 1_{S_{\mathcal D}}\), we have \(Z_S\in(0,\infty)\). By uniform boundedness,
pointwise convergence \(s_\tau\to\mathbf 1_{S_{\mathcal D}}\) \(\pi_0\)-a.e., and dominated
convergence,
\[
 Z_\tau\to Z_S .
\]
For any measurable set \(A\),
\[
 \pi_\tau(A)=Z_\tau^{-1}\int_A s_\tau(\vartheta)\,\pi_0(\mathrm d\vartheta),
 \qquad
 \pi_{\mathcal D}(A)=Z_S^{-1}\int_A
 \mathbf 1_{S_{\mathcal D}}(\vartheta)\,\pi_0(\mathrm d\vartheta).
\]
Therefore
\[
 \|\pi_\tau-\pi_{\mathcal D}\|_{\mathrm{TV}}
 \le
 |Z_\tau^{-1}-Z_S^{-1}|Z_\tau
 +Z_S^{-1}\int |s_\tau-\mathbf 1_{S_{\mathcal D}}|\,\mathrm d\pi_0,
\]
and both terms converge to zero. Hence
\[
 \|\pi_\tau-\pi_{\mathcal D}\|_{\mathrm{TV}}\to0 .
\]

Now fix \(\tau>0\). Since \(s_\tau\) is positive and continuous, there are constants
\(0<c_\tau<C_\tau<\infty\) and \(r_\tau>0\) such that
\[
 c_\tau\le s_\tau(\vartheta)\le C_\tau,
 \qquad \vartheta\in B_{r_\tau}(\pa).
\]
Thus
\[
 \pi_\tau(B_r(\pa))\asymp\pi_0(B_r(\pa)),
 \qquad r\downarrow0,
\]
for each fixed \(\tau>0\). This comparability also preserves the logarithmic index whenever defined. Consequently
\[
 \MI_{\mathrm{pow}}(\pi_\tau,\pa)
 =\MI_{\mathrm{pow}}(\pi_0,\pa),
 \qquad \tau>0 .
\]
If \(A_{S_{\mathcal D},\pa}(r)\asymp r^\xi\ell(r)^\lambda\) with \(\xi>0\), then
Theorem~\ref{thm:acceptance} applied to the hard posterior gives
\[
 d\,\MI_{\mathrm{pow}}^{-1}(\pi_{\mathcal D},\pa)
 =d\,\MI_{\mathrm{pow}}^{-1}(\pi_0,\pa)+\xi .
\]
This differs from the softened exponent
\(d\,\MI_{\mathrm{pow}}^{-1}(\pi_0,\pa)\). Hence
\[
 \MI_{\mathrm{pow}}(\pi_{\mathcal D},\pa)
 \neq
 \lim_{\tau\downarrow0}\MI_{\mathrm{pow}}(\pi_\tau,\pa).
\]

For the strict logarithmic comparison, assume \(\xi=0\), \(\lambda<0\), and that
\(\MI_{\mathrm{log}}(\pi_0,\pa)\) exists and is finite. For each fixed \(\tau>0\), the local
positivity of \(s_\tau\) gives
\[
 \MI_{\mathrm{log}}(\pi_\tau,\pa)
 =\MI_{\mathrm{log}}(\pi_0,\pa).
\]
For the hard posterior, Theorem~\ref{thm:acceptance} gives
\[
 \MI_{\mathrm{log}}(\pi_{\mathcal D},\pa)
 =\MI_{\mathrm{log}}(\pi_0,\pa)+\lambda .
\]
Since \(\lambda<0\), the two finite logarithmic indices differ, and therefore
\[
 \MI_{\mathrm{log}}(\pi_{\mathcal D},\pa)
 \neq
 \lim_{\tau\downarrow0}\MI_{\mathrm{log}}(\pi_\tau,\pa).
\]
The proof is complete.
\end{proof}

\subsection*{Proof of Theorem~\ref{thm:local mass}}
\begin{proof}[Proof of Theorem~\ref{thm:local mass}]
Let \(q=q_{\ll p}+q_{\perp p}\), and write
\[
 R:=\frac{\mathrm d q_{\ll p}}{\mathrm d p}.
\]
For the restrictions of \(q\) and \(p\) to \(E\), define the local Hellinger-type
quantity
\[
 H_E^2(q,p)
 :=p(E)+q(E)-2\int_E\sqrt R\,\mathrm d p.
\]
Using the Lebesgue decomposition, this can be written as
\[
 H_E^2(q,p)
 =\int_E(\sqrt R-1)^2\,\mathrm d p+q_{\perp p}(E).
\]
We first spell out the finite-measure Cauchy--Schwarz step. Since
\[
 \int_E\sqrt R\,\mathrm d p
 \le
 \sqrt{q_{\ll p}(E)}\sqrt{p(E)}
 \le
 \sqrt{q(E)p(E)},
\]
we have
\[
\begin{aligned}
 H_E^2(q,p)
 &=p(E)+q(E)-2\int_E\sqrt R\,\mathrm d p \\
 &\ge p(E)+q(E)-2\sqrt{p(E)q(E)}
 =\left(\sqrt{q(E)}-\sqrt{p(E)}\right)^2 .
\end{aligned}
\]
Hence
\[
 \left|\sqrt{q(E)}-\sqrt{p(E)}\right|
 \le H_E(q,p),
\]
and therefore
\[
 |q(E)-p(E)|
 \le
 \bigl(\sqrt{q(E)}+\sqrt{p(E)}\bigr)H_E(q,p)
 \le 2H_E(q,p),
\]
because \(p(E),q(E)\le1\).

It remains to compare \(H_E^2(q,p)\) with \(D_\alpha^E(q\|p)\). We prove the
scalar inequality
\[
 (\sqrt x-1)^2\le C_\alpha f_\alpha(x),
 \qquad x\ge0,
 \qquad
 C_\alpha:=\max\left\{1,\frac{1-\alpha}{\alpha}\right\}.
\]
Set \(t=\sqrt x\). If \(0<\alpha\le1/2\), then
\(C_\alpha=(1-\alpha)/\alpha\), and the desired inequality is equivalent to
\[
 2\alpha t-t^{2\alpha}-2\alpha+1\ge0 .
\]
The derivative of the left-hand side is
\(2\alpha(1-t^{2\alpha-1})\), which is negative on \((0,1)\) and positive on
\((1,\infty)\). Since the value at \(t=1\) is zero, the inequality follows. If
\(1/2\le\alpha<1\), then \(C_\alpha=1\), and the desired inequality is equivalent
to
\[
 (2\alpha-1)t^2+2(1-\alpha)t\ge t^{2\alpha}.
\]
This follows from weighted AM--GM, because the weights \(2\alpha-1\) and
\(2(1-\alpha)\) are non-negative and sum to one. The endpoint \(x=0\) is covered
by continuity.

Consequently,
\[
 \int_E(\sqrt R-1)^2\,\mathrm d p
 \le C_\alpha\int_E f_\alpha(R)\,\mathrm d p .
\]
Moreover, since \(C_\alpha\ge (1-\alpha)/\alpha\),
\[
 q_{\perp p}(E)
 \le
 C_\alpha\frac{\alpha}{1-\alpha}q_{\perp p}(E).
\]
Therefore
\[
 H_E^2(q,p)
 \le
 C_\alpha\left\{
 \int_E f_\alpha(R)\,\mathrm d p
 +\frac{\alpha}{1-\alpha}q_{\perp p}(E)
 \right\}
 =C_\alpha D_\alpha^E(q\|p).
\]
Combining the two displays yields
\[
 |q(E)-p(E)|
 \le
 2\sqrt{C_\alpha D_\alpha^E(q\|p)}.
\]
This proves the theorem.
\end{proof}
\subsection*{Proof of Corollary~\ref{cor:local mass}}
\begin{proof}[Proof of Corollary~\ref{cor:local mass}]
Apply Theorem~\ref{thm:local mass} with \(E=B_r(\pa)\). If
\[
 D_\alpha^{B_r(\pa)}(q\|p)=O(r^\delta),
\]
then
\[
 |q(B_r(\pa))-p(B_r(\pa))|
 \le
 2\sqrt{C_\alpha D_\alpha^{B_r(\pa)}(q\|p)}
 =O(r^{\delta/2}).
\]
\end{proof}

\subsection*{Proof of Theorem~\ref{thm:local multiplicative}}
\begin{proof}[Proof of Theorem~\ref{thm:local multiplicative}]
Assume \(p(E)>0\). Let \(q=q_{\ll p}+q_{\perp p}\), and write
\[
 R:=\frac{\mathrm d q_{\ll p}}{\mathrm d p},
 \qquad
 x:=\frac1{p(E)}\int_E R\,\mathrm d p,
 \qquad
 h:=\frac{q_{\perp p}(E)}{p(E)}.
\]
Then
\[
 \frac{q(E)}{p(E)}=x+h.
\]
Since
\[
 f_\alpha'(u)=\frac{\alpha}{1-\alpha}(1-u^{\alpha-1})
 \le \frac{\alpha}{1-\alpha},
\]
we have
\[
 f_\alpha(x+h)
 \le
 f_\alpha(x)+\frac{\alpha}{1-\alpha}h .
\]
By Jensen's inequality,
\[
 f_\alpha(x)
 \le
 \frac1{p(E)}\int_E f_\alpha(R)\,\mathrm d p.
\]
Combining these bounds gives
\[
 f_\alpha\!\left(\frac{q(E)}{p(E)}\right)
 \le
 \frac{D_\alpha^E(q\|p)}{p(E)}
 =\overline D_\alpha^E(q\|p).
\]
This proves the first claim.

For the relative error bound, we derive the scalar inequality
\[
 f_\alpha(y)
 \ge
 \frac{\alpha |y-1|^2}{2(1+|y-1|)},
 \qquad y\ge0 .
\]
Since \(f_\alpha(1)=f_\alpha'(1)=0\) and \(f_\alpha''(u)=\alpha u^{\alpha-2}\),
Taylor's formula with integral remainder gives, for \(0\le y\le1\),
\[
 f_\alpha(y)
 =\alpha\int_y^1 (u-y)u^{\alpha-2}\,\mathrm d u
 \ge \frac{\alpha}{2}(1-y)^2
 \ge
 \frac{\alpha(1-y)^2}{2(1+1-y)}.
\]
For \(y=1+s\ge1\), the same formula gives
\[
 f_\alpha(1+s)
 =\alpha\int_0^s (s-u)(1+u)^{\alpha-2}\,\mathrm d u .
\]
Because \(\alpha>0\), \((1+u)^{\alpha-2}\ge(1+u)^{-2}\). Hence
\[
 f_\alpha(1+s)
 \ge
 \alpha\int_0^s \frac{s-u}{(1+u)^2}\,\mathrm d u
 =\alpha\{s-\log(1+s)\}.
\]
The elementary bound
\[
 s-\log(1+s)\ge \frac{s^2}{2(1+s)},
 \qquad s\ge0,
\]
follows by differentiating the difference between the two sides. Therefore the
same scalar inequality holds for \(y\ge1\).

Set
\[
 y:=\frac{q(E)}{p(E)},
 \qquad
 s:=|1-y|,
 \qquad
 z:=\overline D_\alpha^E(q\|p).
\]
The first part gives \(f_\alpha(y)\le z\), so
\[
 \frac{\alpha s^2}{2(1+s)}\le z .
\]
Solving this quadratic inequality for \(s\ge0\) yields
\[
 s
 \le
 \frac z\alpha
 +
 \sqrt{\left(\frac z\alpha\right)^2+\frac{2z}{\alpha}} .
\]
The last display is bounded by
\[
 \frac2\alpha\sqrt{z}\sqrt{z+2}.
\]
Therefore
\[
 \left|1-\frac{q(E)}{p(E)}\right|
 \le
 \frac2\alpha
 \sqrt{\overline D_\alpha^E(q\|p)}
 \sqrt{\overline D_\alpha^E(q\|p)+2}.
\]
This proves the theorem.
\end{proof}
\subsection*{Proof of Corollary~\ref{cor:local multiplicative}}
\begin{proof}[Proof of Corollary~\ref{cor:local multiplicative}]
Set
\[
 t_r:=\overline D_\alpha^{B_r(\pa)}(q\|p).
\]
By assumption,
\[
 t_r=O(r^\delta),
 \qquad r\downarrow0,
\]
and hence \(t_r\to0\). Applying Theorem~\ref{thm:local multiplicative} with \(E=B_r(\pa)\) gives
\[
 \left|1-\frac{q(B_r(\pa))}{p(B_r(\pa))}\right|
 \le
 \frac2\alpha\sqrt{t_r}\sqrt{t_r+2}
 =O(t_r^{1/2})
 =O(r^{\delta/2}).
\]
\end{proof}

\subsection*{Proof of Theorem~\ref{thm:one-sided-MI-pow}}
\begin{proof}[Proof of Theorem~\ref{thm:one-sided-MI-pow}]
Let
\[
 E_r:=B_r(\pa),
 \qquad
 m_p(r):=p(E_r),
 \qquad
 m_q(r):=q(E_r).
\]
By assumption, \(m_p(r)>0\) and \(m_q(r)>0\) for all sufficiently small \(r\).
We use two elementary facts about \(f_\alpha\). First, \(f_\alpha\) is continuous
on \([0,\infty)\), satisfies \(f_\alpha(1)=0\), \(f_\alpha(0)=1\), and diverges
as \(x\to\infty\). Hence, for every \(\lambda<1\), the sublevel set
\[
 \{x\ge0:f_\alpha(x)\le\lambda\}
\]
is contained in an interval \([c,C]\) with \(0<c<C<\infty\). Second, for every
finite \(M\), the sublevel set
\[
 \{x\ge0:f_\alpha(x)\le M\}
\]
is bounded above.

We first prove \textnormal{(i)}. Suppose
\[
 \limsup_{r\downarrow0}\overline D_\alpha^{E_r}(q\|p)<1 .
\]
Choose \(\lambda<1\) such that
\[
 \overline D_\alpha^{E_r}(q\|p)\le\lambda
\]
for all sufficiently small \(r\). By Theorem~\ref{thm:local multiplicative},
\[
 f_\alpha\!\left(\frac{m_q(r)}{m_p(r)}\right)
 \le
 \overline D_\alpha^{E_r}(q\|p)
 \le\lambda .
\]
The first scalar fact gives constants \(0<c<C<\infty\) such that
\[
 c\le\frac{m_q(r)}{m_p(r)}\le C
\]
for all sufficiently small \(r\). Equivalently,
\[
 m_q(r)\asymp m_p(r),
 \qquad r\downarrow0 .
\]
Thus the upper and lower power gauge classes of \(p\) and \(q\) coincide. For
any small-ball function, \(\sup\mathcal L_{\mathrm{pow}}\le
\inf\mathcal U_{\mathrm{pow}}\). Since Assumption~\ref{as:existence} gives equality of these two
quantities for the relevant measures, the common gauge classes identify
\[
 \MI_{\mathrm{pow}}(q,\pa)=\MI_{\mathrm{pow}}(p,\pa).
\]

If \(\MI_{\mathrm{log}}(p,\pa)\) exists, then the common power exponent is
\[
 a:=d\,\MI_{\mathrm{pow}}^{-1}(p,\pa)
 =d\,\MI_{\mathrm{pow}}^{-1}(q,\pa).
\]
Since \(m_q(r)\asymp m_p(r)\), for every \(\eta\in\mathbb R\),
\[
 m_p(r)=O\!\left(r^a(-\log r)^\eta\right)
 \quad\Longleftrightarrow\quad
 m_q(r)=O\!\left(r^a(-\log r)^\eta\right),
\]
and similarly
\[
 m_p(r)=\Omega\!\left(r^a(-\log r)^\eta\right)
 \quad\Longleftrightarrow\quad
 m_q(r)=\Omega\!\left(r^a(-\log r)^\eta\right).
\]
Hence the logarithmic upper and lower gauge classes also coincide. Since the
Logarithmic Mass Index of \(p\) exists, and the common gauge classes force the
same upper and lower logarithmic thresholds for \(q\),
\[
 \MI_{\mathrm{log}}(q,\pa)=\MI_{\mathrm{log}}(p,\pa).
\]

We now prove \textnormal{(ii)}. Suppose
\[
 \limsup_{r\downarrow0}\overline D_\alpha^{E_r}(p\|q)<\infty .
\]
Then there exists \(M<\infty\) such that
\[
 \overline D_\alpha^{E_r}(p\|q)\le M
\]
for all sufficiently small \(r\). Applying Theorem~\ref{thm:local multiplicative} with \(p\) and \(q\)
interchanged gives
\[
 f_\alpha\!\left(\frac{m_p(r)}{m_q(r)}\right)
 \le
 \overline D_\alpha^{E_r}(p\|q)
 \le M .
\]
By the second scalar fact, \(m_p(r)/m_q(r)\) is bounded above. Therefore there
exists \(C<\infty\) such that
\[
 m_q(r)\ge C^{-1}m_p(r)
\]
for all sufficiently small \(r\).

Consequently, every lower power bound for \(m_p\) is also a lower power bound
for \(m_q\):
\[
 \mathcal L_{\mathrm{pow}}(p,\pa)
 \subseteq
 \mathcal L_{\mathrm{pow}}(q,\pa).
\]
Taking suprema and using the dimensional normalisation in Definition~\ref{define:MI} gives
\[
 \sup\mathcal L_{\mathrm{pow}}(q,\pa)
 \ge
 \sup\mathcal L_{\mathrm{pow}}(p,\pa).
\]
Assumption~\ref{as:existence} identifies these lower thresholds with the corresponding Power Mass
Indices, because \(\sup\mathcal L_{\mathrm{pow}}=\inf\mathcal U_{\mathrm{pow}}\)
for both measures. Therefore
\[
 \MI_{\mathrm{pow}}(q,\pa)
 \ge
 \MI_{\mathrm{pow}}(p,\pa).
\]

Finally, assume
\[
 \MI_{\mathrm{pow}}(q,\pa)=\MI_{\mathrm{pow}}(p,\pa)
\]
and that both Logarithmic Mass Indices exist. Let
\[
 a:=d\,\MI_{\mathrm{pow}}^{-1}(p,\pa)
 =d\,\MI_{\mathrm{pow}}^{-1}(q,\pa).
\]
The inequality \(m_q(r)\ge C^{-1}m_p(r)\) implies that every lower logarithmic
bound for \(m_p\) at the common power scale is also a lower logarithmic bound
for \(m_q\). Hence
\[
 \mathcal L_{\mathrm{log}}(p,\pa)
 \subseteq
 \mathcal L_{\mathrm{log}}(q,\pa).
\]
Taking suprema and using the existence of both Logarithmic Mass Indices gives
\[
 \MI_{\mathrm{log}}(q,\pa)
 \ge
 \MI_{\mathrm{log}}(p,\pa).
\]
The proof is complete.
\end{proof}

\section{Additional experimental details}
\label{app:experiment_details}

This appendix records the settings, analytic expressions, and numerical summaries for the experiments in Section~\ref{sec:experiments}. The main text presents the controlled and neural-network curves. The final subsection collects the common interpretation of local RE-KL at an atom. All experiments are finite-radius illustrations; asymptotic Mass Indices follow from the specified families. The code will be provided as an anonymised ZIP archive in the supplementary material.

\subsection{Experiment 1: synthetic calibration}
\label{app:synthetic_calibration}

Experiment~1 plots closed-form small-ball curves on a geometric grid of 300 radii over
\[
r\in[10^{-3},0.6].
\]
The purpose is to visualize representative local regimes rather than to fit a statistical model.

For the regular reference case, we use the one-dimensional Gaussian small-ball mass
\[
p_{\mathrm{Gauss}}(B_r(0))
=
\operatorname{erf}\left(\frac{r}{\sqrt{2}}\right).
\]
This behaves as \(r\) for small \(r\).

For the local-depletion example, we plot
\[
p_{\mathrm{dep}}(B_r(0))=r^2.
\]
This has faster small-radius decay than the regular reference.

For the cusp example, we plot
\[
p_{\mathrm{cusp}}(B_r(0))=\sqrt r.
\]
This has slower small-radius decay than the regular reference.

For the logarithmically corrected example, we plot
\[
p_{\log}(B_r(0))=r(1-\log r).
\]
This has the same leading power order as the regular reference, but differs by a logarithmic factor.

For the spike-slab example, we use spike weight \(0.25\) and slab scale \(0.35\), giving
\[
p_{\mathrm{spike}}(B_r(0))
=
0.25
+
0.75\,
\operatorname{erf}\left(
\frac{r}{0.35\sqrt{2}}
\right).
\]
The nonzero spike produces a small-ball mass that does not vanish as \(r\to0\).

\subsection{Experiment 2: UCI Bayesian logistic-regression sanity check}
\label{app:uci_experiment}

Experiment~2 uses three binary classification datasets: Breast Cancer, Iris \(0\) vs \(1\), and Wine \(0\) vs \(1\). For Iris and Wine, only classes \(0\) and \(1\) are retained. Each dataset is repeated over five random seeds,
\[
\{0,1,2,3,4\}.
\]
The seed controls the stratified train-test split. Conditional on the training set, preprocessing, Laplace optimisation, and quadrature settings are fixed.
For each dataset and seed, we use a stratified train-test split with test fraction \(0.30\). Only the training set is used for the Bayesian logistic-regression fit.

The covariates are standardized on the training set and then projected to four principal components using PCA. An intercept is added, so the parameter dimension is
\[
d=5.
\]
The model is Bayesian logistic regression. Given design matrix \(X\), labels \(y_i\in\{0,1\}\), and parameter \(\pa\in\mathbb R^5\), the negative log posterior objective is
\[
\sum_i
\left[
\log(1+\exp(x_i^\top\pa))
-
y_i x_i^\top\pa
\right]
+
\frac{1}{2\sigma_0^2}\|\pa\|^2,
\qquad
\sigma_0=2.
\]
Equivalently, the prior is
\[
\pa\sim N(0,4I_5).
\]

The posterior is approximated by a Laplace Gaussian approximation. The posterior mode \(\hat\pa\), which is the mean of the Laplace Gaussian, is obtained by minimizing the negative log posterior. We use a trust-region Newton optimizer with analytic gradient and Hessian. If this optimizer fails, we use an L-BFGS-B fallback. The Laplace covariance is the inverse Hessian at \(\hat\pa\), with eigenvalues clipped below \(10^{-10}\) for numerical stability. We set
\[
\pa_0=\hat\pa.
\]

More precisely, \(\hat\pa\) is the posterior mode and hence the mean of the
Laplace approximation. Optimization starts at the zero vector. The primary
\texttt{trust-exact} solver uses gradient tolerance \(10^{-9}\) and at most
200 iterations; the L-BFGS-B fallback uses gradient tolerance \(10^{-10}\),
function tolerance \(10^{-12}\), and at most 1000 iterations, and is accepted
only if its objective is no larger. PCA uses the full SVD solver.

For each run, prior and posterior Euclidean small-ball masses are estimated for
\[
B_r(\pa_0)
=
\{\pa:\|\pa-\pa_0\|_2\leq r\}
\]
over a geometric grid of \(24\) radii in
\[
[0.03,0.45].
\]
The prior covariance is \(4I_5\), while the posterior covariance is the Laplace covariance. The prior mean is \(0\), and the posterior mean is \(\pa_0\).

The small-ball masses are estimated by Sobol quadrature over the Euclidean unit ball. The default number of Sobol points is
\[
2^{15}=32768
\]
per run. The unit-ball Sobol points are generated in dimension \(d+1\). The first \(d\) coordinates are transformed into Gaussian directions and normalized to obtain directions on the sphere. The final coordinate is transformed as \(U^{1/d}\) to obtain the radial component. For a radius \(r\), the quadrature points are
\[
\pa_0+r z_j,
\qquad
z_j\in B_1(0).
\]

For a Gaussian distribution \(N(m,\Sigma)\), the estimated log small-ball mass is
\[
\log |B_1(0)|
+
d\log r
+
\log\left[
\frac{1}{M}
\sum_{j=1}^{M}
\varphi_{\Sigma}(\pa_0+r z_j-m)
\right],
\]
where \(M=32768\), and \(\varphi_{\Sigma}\) denotes the Gaussian density with covariance \(\Sigma\).

The Sobol sequence is scrambled. Indexing runs by \(j=0,\ldots,14\), in the
order Breast Cancer, Iris, and Wine and then seeds \(0,\ldots,4\) within each
dataset, its random seed is \(1009+37j\). Within each run, the same unit-ball
points are reused for the prior and posterior and at every radius.

Finite-radius slopes are computed by first differences on the log-log scale:
\[
s_k
=
\frac{
\log p(B_{r_{k+1}}(\pa_0))
-
\log p(B_{r_k}(\pa_0))
}{
\log r_{k+1}-\log r_k
}.
\]
The slope location is reported at the geometric midpoint
\[
(r_k r_{k+1})^{1/2}.
\]
Table~\ref{tab:uci_slopes} reports small-radius slopes obtained separately
for each run by ordinary least squares with an intercept, regressing log
mass on log radius over the six grid points satisfying \(r\leq0.06\).

\begin{table}[htbp]
\centering\small
\caption{Slopes fitted over \(r\le0.06\): mean \(\pm\) standard error over five seeds.}
\label{tab:uci_slopes}
\setlength{\tabcolsep}{3pt}
\begin{tabular}{lccc}
\toprule
Dataset & \(n\) & Prior & Posterior \\
\midrule
Breast Cancer & 398 & \(5.000\pm0.000\) & \(4.969\pm0.002\) \\
Iris \(0\) vs \(1\) & 70 & \(5.000\pm0.000\) & \(4.999\pm0.000\) \\
Wine \(0\) vs \(1\) & 91 & \(5.000\pm0.000\) & \(4.993\pm0.001\) \\
\bottomrule
\end{tabular}
\end{table}

Across datasets and seeds, the aggregate UCI figure uses
\[
3\times5=15
\]
runs.
For the mass panel, let \(\ell_j(r)\) be the log small-ball mass in run
\(j\). The plotted centre is the geometric mean \(\exp(\bar\ell(r))\),
and the shaded band is
\(\exp[\bar\ell(r)\pm s_\ell(r)/\sqrt{15}]\), where \(s_\ell\) is the
sample standard deviation with denominator 14. Thus the standard error
is computed in log-mass space before exponentiation. The slope panel
instead shows the arithmetic mean of the run-level first-difference
slopes and its standard error. The table uses standard errors over the
five seeds of each dataset; these are not standard deviations or
confidence intervals.

\subsection{Experiment 3: local RE-KL directionality}
\label{app:rekl_directionality}

We use the probability densities \(p(x)=1/2\) and \(q(x)=|x|^{-1/2}/4\) on \((-1,1)\). For \(B_r=(-r,r)\), \(0<r<1\), the masses are \(p(B_r)=r\) and \(q(B_r)=\sqrt r\), with power indices \(1\) and \(2\). These are the measures of Example~\ref{exp:dir}.

For the generator \(f_\alpha\) defined in Section~\ref{subsec:rekl}, direct integration gives
\[
\overline D_\alpha^{B_r}(q\|p)
=1+\frac{\alpha}{1-\alpha}r^{-1/2}
-\frac{2^{-\alpha}}{(1-\alpha)(1-\alpha/2)}r^{-\alpha/2},
\]
and
\[
\overline D_\alpha^{B_r}(p\|q)
=1+\frac{\alpha}{1-\alpha}r^{1/2}
-\frac{2^\alpha}{(1-\alpha)(1+\alpha)}r^{\alpha/2}.
\]
Indeed, \(\int_{B_r}q^\alpha p^{1-\alpha}\,\mathrm dx
=2^{-\alpha}r^{1-\alpha/2}/(1-\alpha/2)\), and the reverse integral is
\(\int_{B_r}p^\alpha q^{1-\alpha}\,\mathrm dx
=2^\alpha r^{(1+\alpha)/2}/(1+\alpha)\).
Substitution into the definition, followed by division by the second measure's ball mass, yields the displayed formulas. Thus
\[
\overline D_\alpha^{B_r}(p\|q)\to1,
\qquad
\overline D_\alpha^{B_r}(q\|p)
\sim\frac{\alpha}{1-\alpha}r^{-1/2}\to\infty.
\]
Figure~\ref{fig:nn_directional}(a) plots these exact expressions at \(\alpha=3/4\) on a geometric grid of 400 radii from \(10^{-4}\) to \(1\). It illustrates Theorem~\ref{thm:one-sided-MI-pow}(ii) with \(\MI_{\mathrm{pow}}(q,0)>\MI_{\mathrm{pow}}(p,0)\), while the hypothesis of part~(i) fails. The analytic panel is generated from these expressions by \texttt{scripts/compact\_figures.py} in the manuscript source.

\subsection{Experiment 4: LeNet-5/MNIST weight marginals}
\label{app:lenet_experiment}

Experiment~4 trains a LeNet-5 Bayesian neural network on MNIST.
The measured distributions are pooled one-dimensional variational weight
marginals, defined precisely in Appendix~\ref{app:nn_protocol}.
The four variational families are Gaussian, Student-\(t\), spike-and-slab,
and hard-concrete, with comparison prior
\(\pi_0=(1-\pi)\delta_0+\pi S\), where \(1-\pi>0\).
The Student-\(t\) family is the tractable scale-mixture representative
listed in Table~\ref{table:MI-common}; it is not a horseshoe distribution.

\paragraph{Marginal quantities and atom masses.}
Small-ball masses and normalised local RE-KL are evaluated for the
equal-weight mixture of selected weight marginals on a geometric radius
grid. Hard-concrete curves use the effective active-component approximation
specified in Appendix~\ref{app:nn_protocol}. The curves are shown in
Figures~\ref{fig:nn_smallball}(a) and~\ref{fig:nn_directional}(b).
For spike-and-slab the atom mass is \(1-\gamma\), and for hard-concrete
it is \(\mathbb P(z=0)\). Table~\ref{tab:nn_mi} supplies the numerical
values summarised in Figure~\ref{fig:nn_smallball}(d).
The Power Mass Indices follow analytically from the marginal families
by Proposition~\ref{prop:basic-calibrations}, rather than from fitting
the finite radius grid.

\begin{table}[htbp]
\centering
\caption{LeNet-5/MNIST marginal Power Mass Indices and reported zero-atom masses.}
\label{tab:nn_mi}
\begin{tabular}{lcc}
\toprule
Variational family & \(\MI_{\mathrm{pow}}\) & Atom mass \\
\midrule
Gaussian & \(1\) & 0 \\
Student-\(t\) & \(1\) & 0 \\
Spike-and-slab & \(\infty\) & 0.52 \\
Hard-concrete & \(\infty\) & 0.30 \\
\bottomrule
\end{tabular}
\end{table}

\subsection{Experiment 5: ResNet channel gates}
\label{app:nn_experiment}

Experiment~5 fixes the pretrained base parameters of ResNet-56/CIFAR-10
and ResNet-50/ImageNet and trains a variational gate for each selected
channel. The variable \(z\) multiplies the convolution output before
BatchNorm, so \(z=0\) zeros that convolution output. BatchNorm running
statistics continue to update during training. Gate zeros do not by
themselves imply a zero post-BatchNorm activation or physical channel
removal. The four families and the atomic spike-and-slab form of the
comparison prior are the same as in Experiment~4; the slab centre and
weight differ as specified in Table~\ref{tab:nn_training_protocol}.

Small-ball mass and normalised local RE-KL are evaluated for the
equal-weight mixture of gate marginals on a geometric radius grid.
All four families are reported
on both ResNets. ImageNet quantities average three seeds, while CIFAR-10
uses one seed. The main-text mass curves appear in
Figure~\ref{fig:nn_smallball}(b--c), and the directional curves in
Figure~\ref{fig:nn_directional}(c--d).

\paragraph{Atom-mass summaries.}
Table~\ref{tab:nn_resnet} gives the values underlying the ResNet entries
in Figure~\ref{fig:nn_smallball}(d). As for the weight marginals, a
positive continuous density at zero gives \(\MI_{\mathrm{pow}}=1\),
and a positive atom gives \(\MI_{\mathrm{pow}}=\infty\).
An atom probability need not equal the realised fraction of pruned
channels in a deterministic network.

\begin{table}[t]
\centering
\caption{ResNet channel-gate Power Mass Indices and reported zero-atom
masses. ImageNet values average three seeds; the reported summaries
have identical rounded values across those seeds and atom-mass standard
errors below \(10^{-4}\). CIFAR-10 uses one seed.}
\label{tab:nn_resnet}
\begin{tabular}{lccc}
\toprule
Variational family & \(\MI_{\mathrm{pow}}\) &
CIFAR-10/ResNet-56 & ImageNet/ResNet-50 \\
\midrule
Gaussian & \(1\) & 0 & 0 \\
Student-\(t\) & \(1\) & 0 & 0 \\
Spike-and-slab & \(\infty\) & 0.32 & 0.065 \\
Hard-concrete & \(\infty\) & 0.03 & 0.023 \\
\bottomrule
\end{tabular}
\end{table}

\paragraph{Detailed ImageNet directional values.}
At the smallest reported radius, \(r=10^{-3}\), with \(\alpha=1/2\),
\(\overline D_\alpha^{B_r}(\pi_0\|q)\) is \(4.4\times10^6\),
\(6.3\), and \(24.0\) for Gaussian, spike-and-slab, and hard-concrete
gates, respectively. In the other direction, the values for both
continuous families differ from \(1\) by less than \(10^{-3}\), while
spike-and-slab and hard-concrete give \(0.51\) and \(0.69\).
These are means over the three seeds. For example, the hard-concrete
prior-to-approximation values range from \(23.9468\) to \(23.9592\).
The large ResNet divergence reflects concentration of continuous gates
near \(1\), which suppresses their mass near zero.

\subsection{Training and marginal-evaluation protocol for Experiments 4--5}
\label{app:nn_protocol}

\paragraph{Configuration and provenance.}
Table~\ref{tab:nn_training_protocol} specifies the supplementary code's
retraining protocol. The original MNIST/CIFAR-10 marginal inputs and
complete training records are absent, so these settings do not establish
exact recovery of their reported results. ImageNet records confirm the
families, seeds, epochs, phase~1, prior weight, penalty coefficient,
CUDA execution, and coordinate count; other training settings are
specified by the shared code and reproduction template.

\begin{table}[t]
\centering\small
\caption{Supplied neural-network reproduction protocol. Penalty weights
are the final coefficients \(\lambda\) before the dataset-size scaling
specified in the text; they are distinct from the RE-KL order \(\alpha\).}
\label{tab:nn_training_protocol}
\setlength{\tabcolsep}{5pt}
\renewcommand{\arraystretch}{1.12}
\begin{tabular}{lccc}
\toprule
Setting & LeNet/MNIST & ResNet-56/CIFAR-10 & ResNet-50/ImageNet \\
\midrule
Training seed(s) & 1 & 1 & 1, 2, 3 \\
Epochs / batch size & 15 / 128 & 8 / 128 & 8 / 192 \\
Optimizer & Adam & SGD & SGD \\
Initial learning rate & \(10^{-3}\) & \(0.05\) (gates) & \(0.05\) (gates) \\
Learning-rate schedule & Constant & Cosine, 8 epochs & Cosine, 8 epochs \\
Penalty anneal / warmup epochs & 5 / 0 & 4 / 0 & 4 / 0 \\
Internal validation fraction & \(0.10\) & \(0.05\) & \(0.05\) \\
Split seed & 42 & 42 & 42 \\
Prior slab weight \(\pi\) & \(0.5\) & \(0.2\) & \(0.2\) \\
Gaussian slab \(S\) & \(N(0,1)\) & \(N(1,1)\) & \(N(1,1)\) \\
\(\lambda\), Gaussian & \(10^{-3}\) & \(10^{-3}\) & \(10^{-4}\) \\
\(\lambda\), Student-\(t\) & \(10^{-3}\) & \(10^{-3}\) & \(10^{-4}\) \\
\(\lambda\), spike-and-slab & \(10^{-3}\) & \(10^{-2}\) & \(10^{-4}\) \\
\(\lambda\), hard-concrete & \(0.1\) & \(10^{-2}\) & \(10^{-4}\) \\
Coordinates in each pooled marginal & 50,000 weights & 1,008 gates & 7,552 gates \\
Radius interval, 24 geometric points & \([10^{-3},0.3]\) & \([10^{-3},1.2]\) & \([10^{-3},1.2]\) \\
RE-KL order \(\alpha\) & \(1/2\) & \(1/2\) & \(1/2\) \\
\bottomrule
\end{tabular}
\end{table}

\paragraph{Models, gate placement, and frozen parameters.}
The MNIST model is LeNet-5-Caffe: a \(5\times5\) convolution with
20 output channels, ReLU and \(2\times2\) max pooling, a
\(5\times5\) convolution with 50 channels, ReLU and the same pooling,
then fully connected layers of widths \(800\to500\to10\), with a
ReLU between the two linear layers. All 430,500 convolutional and linear
weights use the selected variational family. Biases are ordinary learned
point parameters with no KL penalty and are excluded from the marginal
statistics.

The CIFAR-10 base is the ResNet-56 implementation and pretrained
\texttt{cifar10\_resnet56.pth} supplied with the code. The ImageNet base is
Torchvision ResNet-50 with \texttt{IMAGENET1K\_V1} pretrained weights.
Gates multiply the output channels of \texttt{conv1} in each BasicBlock,
and of \texttt{conv1} and \texttt{conv2} in each Bottleneck. The stem,
downsample branches, and last convolution in each residual block are
ungated. This gives 27 gate groups and 1,008 coordinates in ResNet-56,
and 32 groups and 7,552 coordinates in ResNet-50. Base parameters have
gradients disabled, while the model remains in training mode and
BatchNorm running statistics are updated. Only ImageNet phase~1 gate
training is used here; phase~2 knowledge-distillation fine-tuning in the
code is outside these results.

\paragraph{Data preparation and splits.}
MNIST images are converted to tensors and normalized by mean \(0.5\)
and standard deviation \(0.5\), without augmentation. CIFAR-10 training
uses random \(32\times32\) crops after four-pixel padding and random
horizontal flips, then channel means \((0.4914,0.4822,0.4465)\) and
standard deviations \((0.2470,0.2435,0.2616)\); test images use only
tensor conversion and this normalization. ImageNet training uses
\texttt{RandomResizedCrop}(224) and random horizontal flips, with means
\((0.485,0.456,0.406)\) and standard deviations
\((0.229,0.224,0.225)\). Its official evaluation images are resized to
256 pixels on the shorter side and centre-cropped to \(224\times224\)
before normalization. Unspecified transform arguments use Torchvision
defaults in the chosen software environment.

Internal validation subsets are drawn from the training data with split
seed 42 and share its transforms. The entry points use fixed training
epochs and extract final parameters, without validation-based early
stopping or model selection. ImageNet's official evaluation set is
separate; the loader expects class-organized \texttt{train/} and
\texttt{validation/} directories. Python, NumPy, and Torch generators
use the run seed; CUDA runs set deterministic cuDNN behavior and
disable benchmarking.

\paragraph{Training objectives and optimization.}
Let \(\mathcal L_{\rm CE}\) be cross-entropy averaged over a minibatch,
\(R\) the sum of coordinatewise regularizers, and \(N_{\rm train}\)
the number of examples after the internal validation split. The losses are
\[
 \mathcal L_{\rm MNIST}
 =\mathcal L_{\rm CE}+\frac{\beta_e}{N_{\rm train}}R,
 \qquad
 \mathcal L_{\rm ResNet}
 =\mathcal L_{\rm CE}+\beta_e R,
 \qquad
 \beta_e=\lambda\min\!\left\{1,\frac{e+1}{T}\right\},
\]
where \(e=0,1,\ldots\) and \(T\) is the anneal length in the table.
In particular, the structured gate penalty is not divided by dataset
size. Each training forward pass draws one reparameterized weight or
gate sample per coordinate, shared by the examples in that minibatch.

For Gaussian and Student-\(t\) families, \(R\) is the single-sample
estimate of \(\sum_i\mathbb E_{q_i}[\log q_i(W_i)-\log p_c(W_i)]\),
where \(p_c\) is the prior's continuous-part density and the sampled
weights are reused from the forward pass. For spike-and-slab,
\[
 R=\sum_i\left[
 D_{\rm KL}\!\left({\rm Bern}(\gamma_i)\|{\rm Bern}(\pi)\right)
 +\gamma_i D_{\rm KL}\!\left(N(\mu_i,\sigma_i^2)\|S\right)
 \right].
\]
The Gaussian slab KL is evaluated analytically. The hard-concrete
regularizer is instead \(R=\sum_i\mathbb P(z_i>0)\), the expected
\(L_0\) complexity. It does not evaluate KL against the comparison
prior. Consequently, the hard-concrete training objective is an
\(L_0\)-regularized objective, while the common spike-and-slab prior is
used to compare the resulting marginal diagnostics.

MNIST uses Adam with its default moment parameters \((0.9,0.999)\),
\(\epsilon=10^{-8}\), and zero weight decay. ResNet gate optimization
uses SGD with momentum \(0.9\), Nesterov acceleration, and zero gate
weight decay; the cosine schedule has period equal to the eight
training epochs and minimum learning rate zero. The base learning rate
and base weight-decay options do not act on the frozen base parameters.
CUDA training enables Torch automatic mixed precision and gradient
scaling; CPU execution does not.

\paragraph{Variational parameterization and initialization.}
The continuous families are Gaussian and location-scale Student-\(t\)
with fixed degrees of freedom \(\nu=3\). Their scales, and the
spike-and-slab slab scales, are parameterized as
\(\sigma_i=\operatorname{softplus}(\rho_i)+10^{-8}\).
For MNIST, their means are initialized as \(0.01\epsilon_i\),
\(\epsilon_i\sim N(0,1)\), their scale parameters are initialized at
\(0.05\), and the spike-and-slab inclusion probabilities start at
\(\gamma_i=0.5\). For ResNet gates the corresponding means start at
\(1+0.01\epsilon_i\), scales at \(0.1\), and inclusion probabilities
at \(0.95\). The code key \texttt{horseshoe} selects the Student-\(t\)
family described here.

Spike-and-slab training uses a Binary Concrete relaxation with
temperature \(\tau=2/3\): for \(u\sim U(0,1)\), the forward gate is
\(\operatorname{sigmoid}[(\operatorname{logit}\gamma_i+
\log u-\log(1-u))/\tau]\), multiplied by a Gaussian slab draw.
The measured distribution is the discrete-continuous family
\((1-\gamma_i)\delta_0+\gamma_i N(\mu_i,\sigma_i^2)\), not this
continuous training relaxation.

For hard-concrete, write \(\ell_i\) for the learned gate logit and set
\(\tau=2/3\), \(a=-0.1\), and \(b=1.1\). The sampled gate is
\[
 s_i=\operatorname{sigmoid}\!\left(
 \frac{\log u-\log(1-u)+\ell_i}{\tau}\right),
 \qquad
 z_i=\min\{1,\max\{0,a+(b-a)s_i\}\},
 \qquad W_i=\theta_i z_i.
\]
The code clips uniform draws to \([10^{-6},1-10^{-6}]\) for both
relaxations. MNIST initializes \(\theta_i=0.05\epsilon_i\) and
\(\ell_i=0\), learning both thereafter. ResNet fixes \(\theta_i=1\)
and initializes \(\ell_i=2\). The gate's zero probability used in
extraction is
\(p_{0i}=\operatorname{sigmoid}[\tau\log(-a/b)-\ell_i]\).

\paragraph{Coordinate extraction and pooling.}
After MNIST training, the code concatenates marginal parameters from
all convolutional and linear weight tensors, excluding biases, and
selects 50,000 of the 430,500 coordinates uniformly without replacement.
A NumPy generator initialized with the run seed is advanced across
families in the order Gaussian, Student-\(t\), spike-and-slab, and
hard-concrete, so the selected coordinate sets need not coincide.
For ResNet, every gate coordinate is retained. For each family and run,
the plotted distribution is the equal-coordinate mixture
\[
 \bar q=\frac1K\sum_{i=1}^K q_i,
 \qquad
 \bar q(B_r(0))=\frac1K\sum_{i=1}^K q_i([-r,r]),
 \qquad
 \bar q(\{0\})=\frac1K\sum_{i=1}^K q_i(\{0\}).
\]
Thus layers are weighted by their number of retained coordinates,
rather than equally. The RE-KL curves are computed from \(\bar q\)
after mixing densities and atom probabilities. They are not averages
of coordinatewise RE-KL values, which are generally different because
divergence is nonlinear in its arguments.

\paragraph{Effective hard-concrete marginal used for plotting.}
The hard-concrete plotting distribution preserves the extracted zero
probability but replaces the active component by a Gaussian:
\[
 \widetilde q_i=p_{0i}\delta_0+(1-p_{0i})N(\theta_i,s_i^2),
 \qquad
 s_i=\begin{cases}
 \max\{0.02,0.5|\theta_i|\},&\text{MNIST},\\
 0.15,&\text{ResNet, with }\theta_i=1.
 \end{cases}
\]
These distributions are substituted for \(q_i\) in the pooling step.
They are effective marginal approximations, not the exact law of the
clipped hard-concrete variable: the active component's support and
upper-endpoint atom are not retained. The reported hard-concrete
small-ball and RE-KL values therefore describe this explicit proxy;
the zero-atom probabilities themselves are obtained from the gate
parameters. This distinction is especially relevant away from the
small-radius limit.

\paragraph{Small-ball masses and local RE-KL.}
All neural curves use the 24-point geometric grids in the table,
centred at zero. Gaussian and Student-\(t\) small-ball probabilities
are CDF differences at \(-r\) and \(r\); spike-and-slab and the
hard-concrete proxy add the atom to the weighted Gaussian CDF
difference. The comparison prior is
\(p=(1-\pi)\delta_0+\pi N(m,1)\), with \(m=0\) for MNIST and
\(m=1\) for ResNet.

Unlike Experiment~3, neural RE-KL uses \(\alpha=1/2\). Write
\(p=a_p\delta_0+p_c(w)\,\mathrm dw\) and
\(\bar q=a_q\delta_0+q_c(w)\,\mathrm dw\), where the continuous
densities include their mixture weights. At this order the common,
unnormalized numerator in both directions is
\[
 H_r=\int_{-r}^{r}
       \bigl(\sqrt{q_c(w)}-\sqrt{p_c(w)}\bigr)^2\,\mathrm dw
       +(\sqrt{a_q}-\sqrt{a_p})^2,
 \qquad
 \overline D_{1/2}^{B_r}(\bar q\|p)=\frac{H_r}{p(B_r)},
 \quad
 \overline D_{1/2}^{B_r}(p\|\bar q)=\frac{H_r}{\bar q(B_r)}.
\]
The code evaluates the continuous part using the equivalent
\(f_{1/2}\)-density-ratio integrands and SciPy adaptive quadrature on
\([-r,r]\), splitting at zero, with subdivision limit 100 and default
absolute and relative tolerances for the recorded SciPy version.
The atom contribution is evaluated
separately; when \(a_q=0\), the reverse-direction singular penalty
at this order is \(a_p\). In particular, \(H_r\) is symmetric but the
two normalized quantities have different denominators.

For ImageNet, this calculation is performed separately for each seed
and the resulting masses, divergences, and atom probabilities are
averaged arithmetically over seeds \(1,2,3\). Distributions from
different seeds are not mixed before evaluating divergence. Sample
standard deviations use denominator \(3-1\); standard errors divide
these by \(\sqrt3\). The supplied multi-seed plotting script exports
standard deviations, whereas Table~\ref{tab:nn_resnet} describes standard errors.
The neural panels display mean curves without uncertainty bands;
MNIST and CIFAR-10 each have one supplied run configuration.

The table's Power Mass Indices are analytic family-based values. An
auxiliary code export averages the first four successive log-log
slopes and reports the reciprocal, or infinity below a slope of
\(0.15\). This finite-radius heuristic does not establish the
reported analytic indices.

\paragraph{Reproduction sequence and available inputs.}
The archive separates training, marginal extraction, and plotting. The
extraction and plotting scripts below are in the neural-network
\texttt{experiment\_scripts/} directory.
\begin{enumerate}[leftmargin=*,itemsep=2pt]
\item For MNIST, \texttt{train\_marginals.py} trains the four families
and saves their selected weight parameters. Run
\texttt{local\_mass\_nn.py} with \texttt{--mode trained} and that
\texttt{.npz} file, using \(\pi=0.5\), slab centre 0, and maximum
radius 0.3.
\item For CIFAR-10, \texttt{structured\_marginals.py} with
\texttt{--arch resnet56} trains gates on the supplied base and saves
all gate parameters. The same trained-mode plotting entry uses
\(\pi=0.2\), slab centre 1, and maximum radius 1.2.
\item For the supplied ImageNet results, the three files
\path{data/marginals/imagenet_marginals_s1.npz} through
\texttt{imagenet\_marginals\_s3.npz} contain all four families.
\texttt{imagenet\_multiseed.py} reads them directly; regenerating these
diagnostics requires neither ImageNet images nor a training checkpoint.
\item To retrain ImageNet gates, run \texttt{run\_e1\_structured.py}
in the neural-network directory for each family and seed, with
\texttt{--arch resnet50 --dataset imagenet --phase 1 --pretrained}
and the table's configuration. Then use
\texttt{structured\_marginals.py --arch resnet50} to extract each
seed's checkpoints, followed by \texttt{imagenet\_multiseed.py}.
\end{enumerate}
The default \texttt{synthetic} plotting mode is a demonstration;
paper reproduction requires \texttt{--mode trained}. Full commands
are in the neural-network README. MNIST/CIFAR-10 commands produce new
runs; saved ImageNet marginals support direct plotting without the
excluded large checkpoints. Training-record \texttt{budget} and
empirical-MI fields are separate auxiliary measurements.

\paragraph{Software and verification scope.}
Recorded subsequent checks used macOS arm64/CPU, Python 3.11.16,
NumPy 2.4.6, SciPy 1.17.1, Matplotlib 3.11.1, scikit-learn 1.9.0,
Torch 2.13.0, and Torchvision 0.28.0. They cover saved-input plotting,
ImageNet marginal extraction, and synthetic-batch training tests,
not full neural retraining. Original ImageNet records identify
Python 3.10.4 and one GPU per run, but omit its model and complete
Torch/CUDA versions. The supplied requirements pin direct dependencies.

\FloatBarrier

\subsection{Local RE-KL at an atom: interpretation of Experiments 4--5}
\label{app:nn_atom_interpretation}

Both neural experiments compare the pooled one-dimensional marginal
\(q\) defined in Appendix~\ref{app:nn_protocol} with the prior \(\pi_0\).
For hard-concrete, \(q\) denotes the stated effective marginal approximation.
For a continuous marginal \(q\), the
atomic contribution to \(D_\alpha^{B_r}(q\|\pi_0)\) is
\(\pi_0(\{0\})f_\alpha(0)\). The continuous contribution vanishes as
\(r\downarrow0\), since
\(f_\alpha(u)\le1+\alpha u/(1-\alpha)\). Dividing by
\(\pi_0(B_r)\to\pi_0(\{0\})>0\) gives
\[
\overline D_\alpha^{B_r}(q\|\pi_0)\to1.
\]
In the reverse direction, the singular penalty gives
\[
\overline D_\alpha^{B_r}(\pi_0\|q)
\ge \frac{\alpha}{1-\alpha}
\frac{\pi_0(\{0\})}{q(B_r)}
\longrightarrow\infty.
\]
Thus the strict threshold in
Theorem~\ref{thm:one-sided-MI-pow}(i) cannot be replaced by \(\le1\).
Marginals retaining a positive atom keep both directions bounded;
boundedness alone does not imply the subcritical hypothesis of part~(i).
Likewise, a value below \(1\) at a finite radius is consistent with,
but does not establish, a limit superior below \(1\).
These statements explain atomic mismatch in marginal prior comparisons;
the experiments do not measure joint network small-ball probabilities
or approximation error relative to the exact posterior.

\FloatBarrier
\end{document}